%% file: nips-adathink.tex
\documentclass{article}

\usepackage[preprint]{neurips_2026}

\usepackage[utf8]{inputenc}
\usepackage[T1]{fontenc}
\usepackage{lmodern}
\usepackage{hyperref}
\usepackage{url}
\usepackage{booktabs}
\usepackage{amsfonts}
\usepackage{amsmath}
\usepackage{amssymb}
\usepackage{amsthm}
\usepackage{nicefrac}
\usepackage{microtype}
\usepackage{xcolor}
\usepackage{colortbl}
\usepackage{graphicx}
\usepackage{subcaption}
\usepackage{multirow}
\usepackage{algorithm}
\usepackage{algpseudocode}
\usepackage{tablefootnote}
\usepackage{enumitem}
\usepackage{tikz}
\usetikzlibrary{shapes,arrows,positioning,fit,calc}
\usepackage{wrapfig}

\definecolor{improve}{RGB}{0,128,0}
\definecolor{decline}{RGB}{200,0,0}
\definecolor{easycolor}{RGB}{76,175,80}
\definecolor{medcolor}{RGB}{255,152,0}
\definecolor{hardcolor}{RGB}{244,67,54}
\definecolor{nothinkblue}{RGB}{31,119,180}
\definecolor{thinkorange}{RGB}{255,127,14}

\newcommand{\up}[1]{\textcolor{improve}{\textbf{+#1}}}

\newcommand{\ci}[2]{{\scriptsize[#1,\,#2]}}

\newcommand{\method}{\textsc{Mrsd}}
\newcommand{\fullmethod}{Split-Budget Generation via Multi-Round Self-Distillation}
\newcommand{\town}{\textsc{Town}}

\newcommand{\couplingtax}{\emph{coupling tax}}
\newcommand{\splitbudget}{\emph{split-budget generation}}

\newtheorem{theorem}{Theorem}
\newtheorem{proposition}{Proposition}
\newtheorem{corollary}{Corollary}
\newtheorem{definition}{Definition}
\newtheorem{assumption}{Assumption}
\theoremstyle{remark}
\newtheorem{remark}{Remark}

\graphicspath{{paper_figures/}{figures/}}

\title{The Coupling Tax: How Shared Token Budgets Undermine Visible Chain-of-Thought Under Fixed Output Limits}

\author{
\textbf{Wenhua Nie \quad Junlin Liu \quad Jianan Wu \quad Zijie Meng}\\
\textbf{Yilong Fan \quad Zhang Zijian \quad Haoran Zheng}\\
\textbf{Jyh-Shing Roger Jang}\\[3pt]
\normalfont Correspondence: Wenhua Nie, National Taiwan University\\
\normalfont \texttt{d13944014@ntu.edu.tw}
}

\begin{document}

\maketitle

\begin{abstract}
Chain-of-thought reasoning is often treated as a monotone way to improve language-model accuracy by letting a model think longer.
We identify a countervailing effect, the \couplingtax{}: when reasoning traces and final answers share one output-token budget, long traces can crowd out the answer they are meant to support.
Across GSM8K, MATH-500, and five BIG-Bench Hard tasks with Qwen3 models at three scales, non-thinking mode matches or outperforms thinking mode on GSM8K and MATH-500 at every budget up to 2048 tokens, while harder tasks shift the crossover to larger budgets.
We derive a truncation-waste decomposition,
$\mathrm{Acc}_{\mathrm{think}}(b)=\alpha_c F_L(b)+\alpha_t(1-F_L(b))$, that predicts this crossover from chain-length and accuracy statistics and explains inverse scaling within the Qwen family.
A DeepSeek-R1-Distill-Llama-8B replication shows the same pattern under a different thinking interface.
As a mitigation, \splitbudget{} decouples reasoning and answer budgets; on full MATH-500, IRIS reaches 74.0\% accuracy, a strengthened extraction variant reaches 78.8\%, and a fixed non-oracle SC+IRIS gate reaches 83.6\%.
The results show that test-time reasoning should be evaluated as a budget-allocation problem, not only as a question of whether longer traces are available.
\end{abstract}

\section{Introduction}
\label{sec:intro}

\input{sections/introduction_final}

\section{Background and Related Work}
\label{sec:related}

\input{sections/related_final}

\section{The Thinking Tax: An Empirical Study}
\label{sec:thinking-tax}

\input{sections/analysis_final}

\section{Diagnostic Analysis}
\label{sec:theory}

\input{sections/theory_final}

\section{An Exploratory Mitigation: \method{}}
\label{sec:method}

\input{sections/method_final}

\section{Experiments}
\label{sec:experiments}

\input{sections/experiments_final}

\section{Discussion and Conclusion}
\label{sec:conclusion}

\input{sections/conclusion_final}



\clearpage
\bibliographystyle{plainnat}
\bibliography{references}

\appendix
\input{sections/appendix_final}

\end{document}

%% file: sections/introduction_final.tex

\begin{figure}[t]
    \centering
    \includegraphics[width=\linewidth]{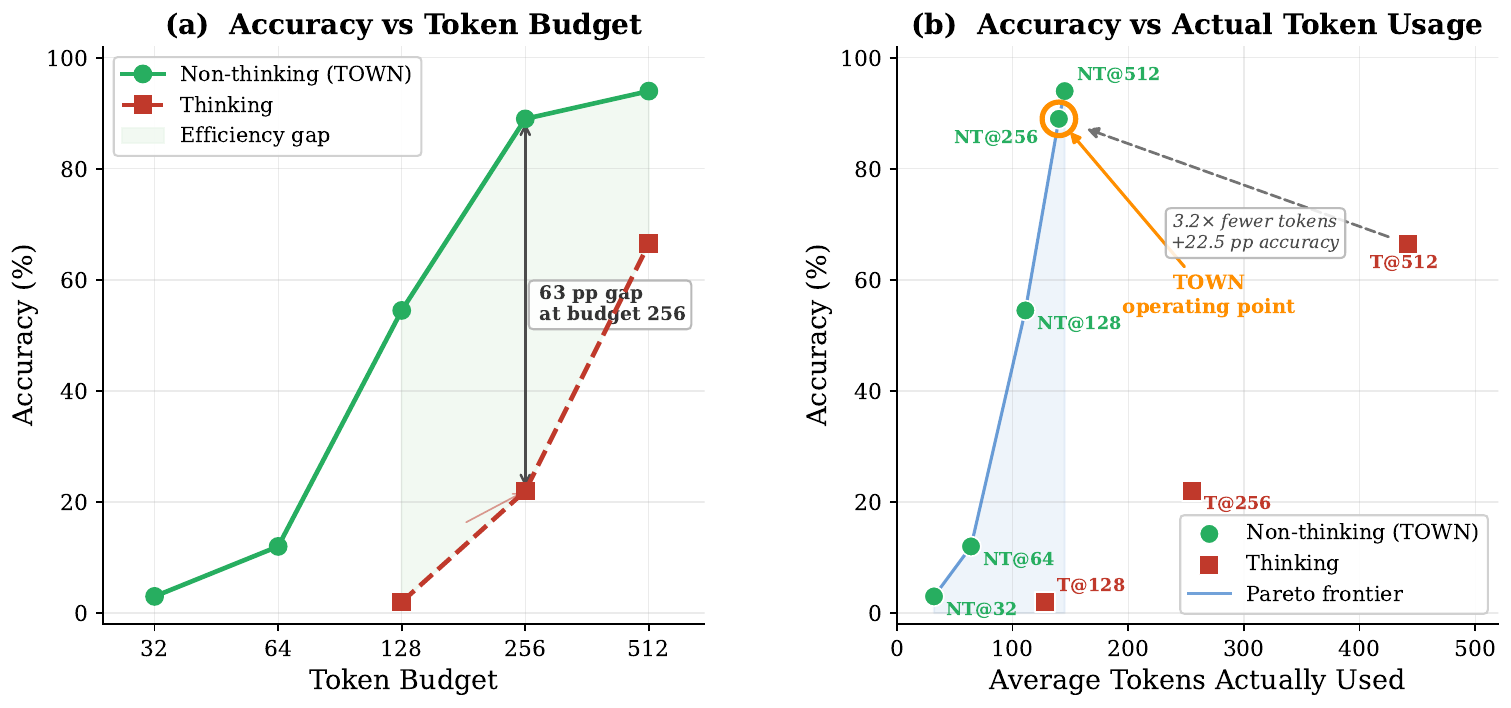}
    \caption{\textbf{The Coupling Tax.}
    Non-thinking mode (blue) dramatically outperforms thinking mode (orange) at every matched token budget $\leq$512 on GSM8K (Qwen3-8B, $n{=}1{,}319$).
    At budget 512, \texttt{nothink@512} achieves \textbf{93.1\%} while \texttt{think@512} reaches just \textbf{56.9\%}---a \textbf{+36.2\,pp} gap.
    The gap widens for the 27B model (Table~\ref{tab:model-size-scaling}).}
    \label{fig:nothink-vs-thinking}
\end{figure}

Chain-of-thought (CoT) reasoning~\citep{wei2022chain} has become the dominant paradigm for improving LLM performance on complex tasks, operationalized in systems like OpenAI o1~\citep{openai2024o1}, QwQ~\citep{qwq2024}, and DeepSeek-R1~\citep{deepseekr1}.
The implicit promise is simple: more thinking tokens, more accuracy.

We study this promise under fixed output-token budgets, the serving knob that directly controls latency, cost, and throughput.
The question is not whether cutting off a chain mid-sentence is harmful.
It is whether a system should invoke visible CoT at all when the answer must fit inside a fixed cap and the model also exposes a native direct-answer mode.
This choice must be made before the full chain length is known, and common deployment caps can fall below a task- and model-specific crossover.
The promise breaks down under two measurable conditions: (i)~reasoning chains typically exceed the budget, causing truncation waste; and (ii)~the model has an architecturally distinct non-thinking mode.
Because autoregressive decoding spends a forward step per generated token, simply raising the cap until every chain completes can multiply latency and throughput cost; the relevant question is therefore the accuracy--token tradeoff at the cap actually served.

\paragraph{The thinking tax.}
Through experiments on GSM8K, MATH-500, and five BBH tasks with three Qwen-family sizes (8B, 9B, 27B), we characterize a phenomenon also observed concurrently by \citet{ma2025nothinking}: \textbf{below the crossover budget, non-thinking mode outperforms thinking mode at every tested budget, often by a wide margin} (Figure~\ref{fig:nothink-vs-thinking}).
At budget 512, non-thinking achieves \textbf{93.1\%} using 152 avg tokens, while thinking reaches just \textbf{56.9\%} with 460 tokens.
On MATH-500, a same-H800 run at budget 2048 gives nothink@2048 = \textbf{68.4\%} vs.\ think@2048 = \textbf{54.8\%}, confirming that the below-crossover tax persists on harder math without cross-hardware comparison.
At 27B on GSM8K, the tax also persists at a 4096-token cap (nothink 98.0\% vs.\ think 87.5\%), showing that the issue is not restricted to extremely small budgets.

The root cause is \emph{truncation waste}: 98.6\% of thinking responses are truncated at $b{=}256$.
Three aspects are surprising: the \emph{magnitude} (69.5\,pp at $b{=}256$, far exceeding format overhead); the \emph{amplification with chain length} ($2.1\times$ from 8B to 9B/27B at $b{=}512$ within the Qwen family); and the \emph{natural-stop oracle} (99.0\% accuracy among naturally completing chains).

\paragraph{The coupling tax and split-budget generation.}
The root cause is architectural: reasoning and the final answer are \emph{coupled} in a single output stream.
We call this the \couplingtax{}.
The fix follows naturally: \emph{decouple} via \splitbudget{}.
Generate a reasoning trace (budget $B_r$), then feed it---even if truncated---to a separate non-thinking pass (budget $B_a$).

\paragraph{\method{}: An exploratory probe.}
We instantiate split-budget generation as \method{} (\fullmethod{}), a training-free framework with three components:
(1)~difficulty triage via non-thinking probe (${\sim}$89\% resolved on GSM8K);
(2)~decoupled answer generation from (possibly truncated) reasoning traces; and
(3)~iterative refinement with convergence-based stopping.
In tables, IRIS@$B$ denotes the one-round \method{} instantiation with thinking budget $B$; full \method{} allows up to three rounds.
On full GSM8K ($n{=}1{,}319$), full \method{} reaches 90.9\%, improving over the same-budget non-thinking probe by +3.41\,pp and over the coupled \town{} cascade by +4.93\,pp, while the cheaper one-round IRIS variant remains slightly higher at 91.43\%.
Full-scale MATH-500 evaluation ($n{=}500$, H800) supports the split-budget gain: IRIS@4096 achieves \textbf{74.0\%} \ci{70.0}{77.7}---exceeding nothink@2048 (68.4\%) by +5.6\,pp and coupled think@4096 (71.0\%) by +3.0\,pp.

\paragraph{Contributions.}
While the empirical observation that nothink can outperform thinking is shared with concurrent work~\citep{ma2025nothinking}, our contribution is not the tautology that truncation hurts.
We provide a claim-to-measurement account of \emph{when} visible CoT should be disabled under a fixed cap and how much budget is needed before it becomes competitive:
\begin{enumerate}[nosep,leftmargin=*]
    \item \textbf{Diagnostic Crossover Estimator} (\S\ref{sec:theory}): $\mathrm{Acc} = \alpha_c F_L + \alpha_t(1{-}F_L)$ turns measured chain-length and truncated-chain accuracy into a crossover estimate---absent from concurrent work.
    \item \textbf{Inverse Scaling} (\S\ref{sec:finding:model-size}): an observed $2.1\times$ tax ratio from 8B to 27B at $b{=}512$, a within-family scaling pattern not identified by \citet{ma2025nothinking} or \citet{xu2025elastic}.
    \item \textbf{Training-Free Isolation and Mitigation} (\S\ref{sec:method}, \S\ref{sec:experiments}): IRIS is a deliberately simple cascade that isolates the effect of decoupling reasoning and answering, recovering +25.4\,pp over coupled \town{} on the same escalated MATH-500 samples, without RL fine-tuning (cf.\ Elastic Reasoning~\citep{xu2025elastic}).
    \item \textbf{Cross-Scale Validation}: a 27B multi-seed IRIS--TOWN gap of +34.5\,pp supports the scaling account.
\end{enumerate}

%% file: sections/related_final.tex

\paragraph{Chain-of-Thought Reasoning.}
Chain-of-thought (CoT) prompting~\citep{wei2022chain} and its variants~\citep{wang2023selfconsistency, yao2023tree} have become the de facto approach for improving LLM reasoning.
Recent work has moved beyond prompting to \emph{training} models with explicit reasoning traces, yielding ``thinking'' LLMs such as DeepSeek-R1~\citep{deepseekr1}, QwQ~\citep{qwq2024}, and Qwen3~\citep{qwen3_2025}.
These models produce a \texttt{<think>...</think>} block before the final answer.
While effective at unconstrained budgets, we show that this thinking overhead imposes a severe penalty under realistic token constraints.

\paragraph{Test-Time Compute Scaling.}
The scaling of inference-time computation has emerged as a complementary axis to model scaling~\citep{snell2024scaling}.
Best-of-N sampling with learned verifiers~\citep{cobbe2021gsm8k}, process reward models~\citep{lightman2023prm,wang2024math}, and iterative refinement~\citep{madaan2023selfrefine} all trade additional tokens for improved accuracy.
\citet{muennighoff2025s1} demonstrate that ``budget forcing''---controlling the length of reasoning chains---can improve efficiency, but operate within the thinking paradigm.
Our findings challenge a shared assumption of these approaches: that \emph{more reasoning tokens always help}.
At constrained budgets, the overhead of structured reasoning can actively \emph{hurt} performance.

\paragraph{Adaptive Computation and Early Exit.}
Adaptive computation~\citep{graves2016adaptive, dehghani2019universal} allows models to allocate variable compute per input.
In the LLM setting, speculative decoding~\citep{leviathan2023fast, chen2023accelerating} and early exit mechanisms~\citep{schuster2022confident, bae2023fast} reduce inference cost by terminating generation early.
Our work identifies a \emph{free} confidence signal---natural stopping---that requires no additional training or auxiliary models, and differs from these approaches in routing \emph{between} thinking and non-thinking modes rather than adjusting depth within a single forward pass.
More broadly, LLM cascades~\citep{chen2023frugalgpt} route
queries between models of different sizes to balance cost and quality.
\method{} applies the cascade principle \emph{within} a single model by
routing between reasoning modes rather than model sizes---a distinction
enabled by the hybrid \texttt{think}/\texttt{nothink} architecture of
models like Qwen3.
The connection to cascade classifiers~\citep{viola2001rapid} is deliberate: we formalize \method{} as a two-stage cascade with interpolation-dominance guarantees (Appendix).

\paragraph{Budget-Aware Reasoning and Concurrent Work.}
A growing body of concurrent work addresses reasoning efficiency.
AdaptThink~\citep{adathink2025} teaches models to decide when to think; SelfBudgeter~\citep{li2025selfbudgeter} learns per-instance budgets.
Most closely related, \citet{ma2025nothinking} concurrently demonstrate that reasoning models can be effective \emph{without} thinking, showing nothink outperforms thinking under matched token budgets---an empirical finding our work shares.
Elastic Reasoning~\citep{xu2025elastic} concurrently proposes split-budget generation via a \emph{training-based} approach (budget-constrained RL rollout).
Our work differs from both in three ways:
(i)~we provide a \emph{diagnostic decomposition} ($\mathrm{Acc} = \alpha_c F_L + \alpha_t(1{-}F_L)$) that estimates the crossover budget from measured chain-length statistics---neither Ma~et~al.\ nor Xu~et~al.\ offer this closed-form diagnostic;
(ii)~we identify and quantify \emph{inverse scaling} of the tax with model size ($2.1\times$ amplification from 8B to 27B at $b{=}512$), a finding absent from concurrent work;
(iii)~our IRIS method is \emph{training-free}, exploiting the model's native nothink mode for answer extraction rather than requiring RL fine-tuning.

\paragraph{Overthinking and Compute Waste.}
\citet{chen2024overthinking} identify overthinking where models produce redundant steps.
We study a more fundamental failure mode: \emph{incomplete} reasoning due to truncation, and provide a diagnostic model (Proposition~\ref{prop:acc-decomp}) with quantitative estimates verified empirically.

%% file: sections/analysis_final.tex

Before introducing our method, we present a systematic empirical study.
All experiments use \textbf{Qwen3-8B}~\cite{qwen3_2025} with native
thinking mode, with \textbf{Qwen3.5-9B} and \textbf{Qwen3.5-27B} for
cross-scale validation, evaluated on full GSM8K ($n{=}1{,}319$) unless
noted.
We control test-time compute via \texttt{max\_new\_tokens}, capping
\emph{total} output tokens; both modes operate under the \emph{same}
budget~$b$.
Non-thinking uses \texttt{enable\_thinking=False}.
Greedy decoding ($\tau{=}0$) is used throughout to isolate truncation
from sampling variance.
When thinking mode exhausts its budget, a multi-level heuristic
extracts a numerical answer from the truncated output
(\texttt{\textbackslash boxed\{\}}, final-answer markers, last-number fallback;
see Appendix~\ref{app:experimental_details})---this is \emph{favorable}
to thinking mode.

\subsection{Finding 1: Non-Thinking Beats Thinking at All Matched Budgets}
\label{sec:finding:tax}

\input{paper_figures/table_thinking_tax_main}

On the complete GSM8K test set,
\texttt{nothink@256} achieves \textbf{87.5\%} vs.\ only
\textbf{18.0\%} for \texttt{think@256}---a \textbf{69.5\,pp} gap.
At 27B, the gap is even starker: nothink@512 = 95.5\% vs.\
think@512 = 18.4\% (+77.1\,pp; Table~\ref{tab:model-size-scaling}).
Non-thinking saturates at moderate budgets; thinking reaches parity only at
$b{=}2048$, using 744 vs.\ 153 average tokens.

\paragraph{Failure-mode decomposition.}
The tax is mostly truncation, not weak reasoning conditional on finishing.
At $b{=}256$, 98.6\% of thinking responses hit the cap, while the rare natural-stop samples are all correct.
At $b{=}512$, natural-stop rate rises to 37.4\% and those samples reach 99.0\% accuracy, but the still-truncated majority reaches only 31.8\%.

\subsection{Finding 2: Natural Stop Is a Free Confidence Oracle}
\label{sec:finding:natural-stop}

At $b{=}512$, 37.4\% of samples terminate naturally
(accuracy \textbf{99.0\%}) while truncated samples reach only 31.8\%, a 67.2\,pp gap.
This signal is binary, endogenous, and free---no logit access or
calibration needed.
DeepSeek-R1-8B shows similar natural-stop behavior.


Token utilization analysis (Appendix~\ref{app:utilization}) reveals
substantial waste at higher budgets, motivating adaptive allocation.

\subsection{Finding 4: The Thinking Tax Scales with Model Size}
\label{sec:finding:model-size}

At $b{=}512$, thinking-mode accuracy \emph{collapses} with model size:
56.9\% (8B), 15.5\% (9B), 18.4\% (27B)
while non-thinking remains high:
93.1\%, 93.2\%, 95.5\%.
At $b{=}512$, the tax is 36.2\,pp (8B), 77.7\,pp (9B), 77.1\,pp (27B)---a
$2.1\times$ observed ratio at this budget.
The root cause is longer chains: the 27B natural-stop rate is just
0.7\% at $b{=}512$ (vs.\ 37.4\% for 8B).
At $b{=}1024$, 9B nothink remains 94.6\% \ci{93.3}{95.7}
while think@1024 is 41.8\%; even at $b{=}2048$, 9B thinking reaches only
66.8\%.

\subsection{Finding 5: The Tax Generalizes Beyond Mathematics}
\label{sec:finding:bbh}

The same pattern extends to non-mathematical reasoning:
on the five-task BBH suite, non-thinking has a +33.3\,pp advantage at $b{=}256$, and the crossover occurs between 1024 and 2048 tokens.
Per-task crossovers vary from ${\sim}$512 (\texttt{boolean\_expressions})
to ${>}$2048 (\texttt{causal\_judgement}), consistent with the crossover diagnostic
(Proposition~\ref{prop:crossover}).
Full per-task results in Appendix~\ref{app:bbh}.

\paragraph{Summary.}
Our findings show:
(1)~non-thinking dominates at all matched budgets $\leq$512;
(2)~natural stop predicts 99.0\% accuracy;
(3)~the tax worsens with model scale;
(4)~the tax generalizes to non-mathematical reasoning (BBH); and
(5)~31.8\% of problems are beyond reach at any budget
(Appendix~\ref{app:impossible}).
These findings motivate \method{} (\S\ref{sec:method}).

%% file: paper_figures/table_thinking_tax_main.tex
\begin{table}[t]
\centering
\caption{\textbf{The Thinking Tax: Non-thinking outperforms thinking at all matched budgets.}
Qwen3-8B on full GSM8K ($n{=}1{,}319$). 27B results in Table~\ref{tab:model-size-scaling}.}
\label{tab:thinking-tax}
\small
\begin{tabular}{ll rrrr}
\toprule
\textbf{Config} & \textbf{Budget} & \textbf{Accuracy} & \textbf{Avg Tok} & \textbf{Early Stop} \\
\midrule
\texttt{nothink@256} & 256 & 87.5\%  & 146 & 88.8\% \\
\texttt{nothink@512} & 512 & \textbf{93.1\%} & \textbf{152} & 99.7\% \\
\cmidrule{1-5}
\texttt{think@256}   & 256 & 18.0\%  & 255 & 1.4\% \\
\texttt{think@512}   & 512 & 56.9\%  & 460 & 37.4\% \\
\midrule
\multicolumn{2}{l}{\textbf{Tax @ 256}} & \textbf{69.5\,pp} & --- & --- \\
\multicolumn{2}{l}{\textbf{Tax @ 512}} & \textbf{36.2\,pp} & --- & --- \\
\bottomrule
\end{tabular}
\vspace{-2mm}
\end{table}

%% file: sections/theory_final.tex

We organize the empirical findings into a quantitative diagnostic
decomposition that expresses thinking-mode accuracy, characterizes the
crossover budget, and accounts for inverse scaling through chain-length
statistics.
The goal is not to claim a new probability law; the value is that the
identity exposes which measurable quantities determine the mode choice
under a fixed cap and yields testable crossover predictions.
We validate this use in Appendix~\ref{app:theory-verification}: 20 random
50-sample GSM8K pilots predict budget-sweep accuracy with 3.48\,pp average
RMSE, and a held-out BBH check has 0.8\,pp error at $b{=}2048$.

\subsection{A Diagnostic Model of Truncation Waste}
\label{sec:theory:truncation}

\begin{definition}[Thinking chain length]
\label{def:chain-length}
For model $\mathcal{M}$ and question $q$, let $L(q) \in \mathbb{N}$
denote the \emph{natural chain length} if unconstrained ($b \to \infty$),
with CDF $F_L(t) \triangleq \Pr(L \le t)$.
\end{definition}

\begin{definition}[Truncation rate]
\label{def:truncation}
$\rho(b) \triangleq 1 - F_L(b) = \Pr(L > b)$.
\end{definition}

\begin{assumption}[Binary outcome structure]
\label{asm:binary}
When $L \le b$: accuracy $\alpha_c(b) \triangleq \Pr(\text{correct} \mid L \le b)$.
When $L > b$: residual accuracy $\alpha_t(b) \ll \alpha_c(b)$.
Empirically, $\alpha_c(512) = 99.0\%$, $\alpha_t(512) = 31.8\%$ on GSM8K.
\end{assumption}

\begin{proposition}[Accuracy decomposition]
\label{prop:acc-decomp}
Under Assumption~\ref{asm:binary}:
\begin{equation}
  \mathrm{Acc}_{\mathrm{think}}(b)
  = F_L(b) \cdot \alpha_c
  + \bigl(1 - F_L(b)\bigr) \cdot \alpha_t.
  \label{eq:acc-decomp}
\end{equation}
\end{proposition}

\noindent\emph{Proof.} Law of total probability, conditioning on $\{L \le b\}$.
The value of the decomposition is diagnostic rather than algebraic novelty:
it isolates the two measured quantities, $F_L$ and $\alpha_t$, that govern
when thinking fails under truncation and when increasing the cap should erase
the tax.
\textbf{Consistency check:} at $b{=}512$, $0.374 \times 99.0\% + 0.626 \times 31.8\% = 56.9\%$, matching observed ($n{=}1{,}319$).
On MATH-500 at $b{=}1024$: estimate $18.0\%$, observed $18.0\%$.
Held-out BBH tests give 0.8\,pp error at $b{=}2048$, and 20 random 50-sample GSM8K pilots give 3.48\,pp average RMSE across budgets (Appendix~\ref{app:theory-verification}).

\begin{proposition}[The Thinking Tax]
\label{thm:thinking-tax}
With $\alpha_t \approx 0$:
\begin{equation}
  \mathrm{Acc}_{\mathrm{nt}}(b) - \mathrm{Acc}_{\mathrm{think}}(b) \;\approx\; \mathrm{Acc}_{\mathrm{nt}}(b) - F_L(b) \cdot \alpha_c.
  \label{eq:tax}
\end{equation}
When $F_L(b) \ll 1$, the tax approaches $\mathrm{Acc}_{\mathrm{nt}}(b)$ itself.
\end{proposition}

\subsection{Crossover Budget}
\label{sec:theory:crossover}

\begin{proposition}[Crossover budget]
\label{prop:crossover}
The crossover $b^*$ satisfies:
\begin{equation}
  F_L(b^*) \;=\;
    \frac{\mathrm{Acc}_{\mathrm{nt}}(b^*) - \alpha_t(b^*)}
         {\alpha_c(b^*) - \alpha_t(b^*)}.
  \label{eq:crossover-exact}
\end{equation}
Under the heuristic ($\alpha_t \approx 0$, stable $\alpha_c$):
$b^* \approx F_L^{-1}(\bar{\alpha}_{\mathrm{nt}} / \alpha_c)$.
\label{eq:crossover}
\end{proposition}

\noindent\emph{Proof sketch.} Set $\mathrm{Acc}_{\mathrm{think}}(b^*) = \mathrm{Acc}_{\mathrm{nt}}(b^*)$ in Eq.~\ref{eq:acc-decomp}; full proof in Appendix~\ref{app:theory-full}.
For 27B: $F_L(b^*) \approx 0.965$, requiring ${\sim}$97\% of chains to complete.

\begin{corollary}[Empirical budget multiplier]
\label{cor:multiplier}
$\gamma \triangleq b^*/b_{\mathrm{sat}} \approx 4\times$ on GSM8K; $> 2\times$ on MATH-500.
\end{corollary}

\subsection{Natural Stop as a Confidence Oracle}
\label{sec:theory:oracle}

\begin{proposition}[Oracle precision]
\label{prop:oracle}
$\mathrm{PPV}(\{L \le b\}) = \alpha_c(b)$.
Empirically $99.0\%$ at $b{=}512$; the proxy event $|y|<0.95b$ has a Hoeffding lower bound $\ge 0.907$ ($n{=}475$).
\end{proposition}

\subsection{Inverse Scaling with Model Size}
\label{sec:theory:scaling}

\begin{proposition}[Inverse scaling]
\label{prop:inverse-scaling}
If (i)~$\mathrm{Acc}_{\mathrm{nt}}(b)$ is size-invariant, (ii)~$\alpha_t \approx 0$, and (iii)~$F_{L_{M_2}}(b) \alpha_c(M_2) \le F_{L_{M_1}}(b) \alpha_c(M_1)$ for $M_2 > M_1$, then $\mathrm{Tax}(M_2, b) \ge \mathrm{Tax}(M_1, b)$.
\label{eq:inverse-scaling}
Condition~(iii) follows from stochastic dominance when $\alpha_c$ is model-invariant.
\end{proposition}

\noindent\emph{Proof sketch.} $\mathrm{Tax}(M) \approx \mathrm{Acc}_{\mathrm{nt}} - F_{L_M}\alpha_c$; larger models have smaller $F_{L_M}$ (Appendix~\ref{app:theory-full}).
At $b{=}512$: tax is 36.2\,pp (8B) vs.\ 77.1\,pp (27B), an observed $2.1\times$ ratio at this budget.

\begin{corollary}[Crossover grows with model size]
\label{cor:crossover-scaling}
Under stochastic dominance, $b^*(M_2) \ge b^*(M_1)$ for $M_2 > M_1$.
\end{corollary}

\subsection{The Coupling Tax and Split-Budget Generation}
\label{sec:theory:bridge}
\label{sec:theory:coupling}

The coupling constraint $|Z| + |A| \le b$ forces reasoning and answering to compete for the same budget.
When $\alpha_c \approx \bar{\alpha}_{\mathrm{nt}}$, truncation waste accounts for essentially the entire tax.
The decomposition identifies two levers: increase $F_L(b)$ (larger budgets) or increase $\alpha_t$ by feeding truncated traces to a separate answer pass.
This motivates \splitbudget{}: allocate separate budgets $b_r$ (reasoning) and $b_a$ (answering) with no coupling.

\begin{definition}[Coupled vs.\ split generation]
\label{def:coupled-split}
Coupled: $|Z| + |A| \le b$.\quad
Split: reasoning $b_r$, answering $b_a$, separately budgeted.
\end{definition}

\begin{proposition}[Recoverable coupling tax]
\label{prop:recoverable-tax}
\begin{equation}
  \Delta_{\mathrm{split}}(b_r, b_a)
  = \bigl(1 - F_L(b_r)\bigr) \cdot
  \bigl(\alpha_{\mathrm{extract}}(b_r, b_a) - \alpha_t(b_r)\bigr)
  \ge 0.
  \label{eq:recoverable-tax}
\end{equation}
\end{proposition}

\noindent\emph{Proof sketch.} Split accuracy $= F_L \alpha_c + (1{-}F_L)\alpha_{\mathrm{extract}}$; subtract Eq.~\ref{eq:acc-decomp} (Appendix~\ref{app:theory-full}).
On MATH-500: $\rho(2048) = 0.912$, IRIS--TOWN gap = +12.2\,pp; $\rho(4096) = 0.553$, gap = +2.2\,pp.

\paragraph{Extended theory.}
Appendix~\ref{app:extended-theory} develops: modal specialization (Propositions~\ref{prop:modal-inequality}--\ref{prop:coupling-impossibility}), optimal budget allocation (Proposition~\ref{prop:optimal-split}), tax decomposition (Definition~\ref{def:tax-decomposition}), cross-scale prediction (Proposition~\ref{prop:cross-scale}), and DFR modal dominance (Proposition~\ref{prop:dfr-dominance}).

%% file: sections/method_final.tex

The decomposition in \S\ref{sec:theory} reveals a structural cause for the
thinking tax: reasoning and answering \emph{share a single output channel}.
This section develops \emph{split-budget generation} and instantiates it as
\textsc{Mrsd} (\textbf{M}ulti-\textbf{R}ound
\textbf{S}elf-\textbf{D}istillation), a training-free inference
framework.
\method{} is deliberately simple: it demonstrates that the coupling tax
is recoverable through inference-time restructuring alone, establishing a
lower bound on what more complex approaches can achieve.

\subsection{The Coupling Problem}
\label{sec:method:coupling}

In standard thinking mode, reasoning tokens~$Z$ and answer tokens~$A$
share one budget:
\begin{equation}
  \label{eq:coupling-constraint}
  |Z| + |A| \;\le\; b.
\end{equation}
When $L(q) > b$, generation is truncated mid-chain before any answer
tokens are produced, creating a zero-sum coupling.
At $b{=}256$ on GSM8K, $F_L(256) < 0.02$---over 98\% of chains are
truncated.

\subsection{Split-Budget Generation}
\label{sec:method:split}

\paragraph{Key insight.}
Decouple reasoning from answering by allocating \emph{separate} budgets:
a \textbf{reasoning pass} (budget $B_r$, possibly truncated) and an
\textbf{answer pass} (budget $B_a$) that reads the trace as context.
The total cost is $B_r + B_a$, but truncation no longer prevents a
well-formed answer.
A truncated chain $Z_{1:B_r}$ is not worthless---it contains partial
decomposition and intermediate computations.
The accuracy gain (Proposition~\ref{prop:recoverable-tax}) is:
\begin{equation}
  \label{eq:split-gain}
  \Delta \;=\; \bigl(1 - F_L(B_r)\bigr) \cdot
    \bigl(\alpha_{\mathrm{extract}}(B_r, B_a) - \alpha_t(B_r)\bigr)
  \;\ge\; 0.
  \tag{\ref{eq:recoverable-tax}}
\end{equation}

\subsection{\textsc{Mrsd}: Multi-Round Self-Distillation}
\label{sec:method:mrsd}

\method{} has three stages.
First, a non-thinking probe with budget $B_1$ resolves easy queries by natural stopping and returns immediately.
Second, unresolved queries receive a thinking pass with budget $B_r$; if the chain completes, its answer is used directly, and if it truncates, the partial trace is passed to a separate non-thinking answer extractor with budget $B_a$.
Third, up to $K$ refinement rounds reuse the previous answer as a hint, stop on consecutive agreement, and otherwise fall back to a majority vote over extracted answers.
This is the same algorithm evaluated below; full pseudocode is deferred to Appendix~\ref{app:method-details}.

\paragraph{Stage 0: Difficulty triage.}
A non-thinking probe with budget $B_1$ resolves easy queries via natural
stopping (${\sim}$88.8\% on GSM8K at $B_1{=}256$, accuracy 94.4\%).

\paragraph{Stage 1: Think-then-extract.}
For the ${\sim}$11.2\% of hard queries, split-budget generation produces
a reasoning trace (budget $B_r$), then feeds it to a non-thinking answer
pass (budget $B_a$), eliminating the zero-sum coupling.

\paragraph{Stages $k \ge 2$: Iterative refinement.}
Subsequent rounds provide the previous answer as a hint; the model can
verify, correct, or re-derive.
Convergence (consecutive agreement) stops iteration;
98\% of escalated queries converge within $K{=}3$ rounds on GSM8K.

\noindent
Token efficiency analysis, design choices, and the interpolation-dominance
theorem (Theorem~\ref{thm:mrsd-pareto}) are deferred to
Appendix~\ref{app:method-details}.

%% file: sections/experiments_final.tex

We evaluate \method{} on GSM8K and MATH-500 and compare against
single-mode baselines, the \town{} cascade, and compute-scaled
alternatives.
All experiments use greedy decoding ($\tau{=}0$); a sampling robustness
check ($\tau{=}0.7$) is in Appendix~\ref{app:sampling-robustness}.
Experiment accounting (full budget/sample-size table) is in
Appendix~\ref{app:experiment-accounting}.

\subsection{Setup}
\label{sec:experiments:setup}

\paragraph{Benchmarks.}
(1)~\textbf{GSM8K}~\citep{cobbe2021gsm8k}: 1{,}319 grade-school math ($n{=}200$ pilot, $n{=}1{,}319$ full).
(2)~\textbf{MATH-500}~\citep{hendrycks2021math}: 500 competition-level problems ($n{=}200$ pilot, $n{=}500$ full).

\paragraph{Models.}
Primary: Qwen3-8B. Cross-scale: Qwen3.5-9B and Qwen3.5-27B.

\paragraph{Baselines.}
\textbf{Nothink@$B$}: non-thinking, budget $B$.
\textbf{Think@$B$}: thinking, budget $B$.
\textbf{\town{}}: two-stage cascade (nothink $\to$ thinking).
\textbf{IRIS@$B$}: 1-round \method{} (triage + single think-then-extract, $B_{\text{think}}{=}B$).

\paragraph{\method{} configuration.}
GSM8K: $B_1{=}256$, $B_{\text{think}}{=}512$, $B_{\text{answer}}{=}128$, max rounds${=}3$.
MATH-500: $B_1{=}512$, $B_{\text{think}}{=}1024$, $B_{\text{answer}}{=}256$, max rounds${=}3$.

\subsection{Main Results}
\label{sec:experiments:main}

\begin{table}[t]
\centering
\caption{\textbf{Full GSM8K evaluation} ($n{=}1{,}319$, Qwen3-8B, seed=42).
\method{} significantly improves over the non-thinking probe and coupled \town{} cascade,
while the one-round IRIS extraction baseline remains a strong low-cost alternative.}
\label{tab:gsm8k-mrsd-full}
\small
\setlength{\tabcolsep}{5pt}
\begin{tabular}{lccc}
\toprule
\textbf{Method} & \textbf{Accuracy (\%)} & \textbf{Avg tokens} & \textbf{Paired vs.\ \method{}} \\
\midrule
Nothink@$256$ & 87.49 & 146.31 & \method{} +3.41\,pp, $p=6.1{\times}10^{-9}$ \\
\town{} ($256{\to}512$) & 85.97 & 203.56 & \method{} +4.93\,pp, $p=5.1{\times}10^{-17}$ \\
IRIS-single ($256/512/128$) & \textbf{91.43} & 204.42 & \method{} $-0.53$\,pp, $p=0.14$ \\
\textbf{\method{} (3-round)} & 90.90 & 287.51 & --- \\
\bottomrule
\multicolumn{4}{p{0.94\linewidth}}{\scriptsize Exact paired McNemar tests use per-sample outputs from the full GSM8K run documented in Appendix~\ref{app:town-simulation}.}
\end{tabular}
\end{table}

\begin{table}[t]
\centering
\caption{\textbf{Method comparison on MATH-500.} Full-scale ($n{=}500$) IRIS results support pilot trends. The SC+IRIS row shows that the two inference strategies are complementary.}
\label{tab:compute-matched}
\small
\setlength{\tabcolsep}{3.5pt}
\begin{tabular}{lccc}
\toprule
\textbf{Method} & \textbf{$n$} & \textbf{Total Tokens} & \textbf{MATH-500 (\%)} \\
\midrule
Nothink@512 (pilot) & 200 & 418 & 47.5 \\
Nothink@1024 & 500 & 606 & 59.8 \\
Nothink@2048 & 500 & 585 & 68.4 \\
Nothink@4096$^{\dagger}$ & 500 & 684 & 68.6 \\
Think@2048$^{\dagger}$ & 500 & 1706 & 54.8 \\
Think@4096$^{\dagger}$ & 500 & 2567 & 71.0 \\
\midrule
SC@3 nothink b=512$^{\dagger}$ & 500 & 1182 & 55.8 \\
SC@5 nothink b=512$^{\dagger}$ & 500 & 1970 & 57.0 \\
SC@5 nothink b=1024$^{\dagger}$ & 500 & 2685 & 76.6 \\
\textbf{SC@5 + IRIS+ gate$^{\ddagger}$} & \textbf{500} & \textbf{3499} & \textbf{83.6} \\
\midrule
TOWN@2048 (full) & 500 & 1590 & 55.0 \\
\textbf{IRIS@2048 (full)} & \textbf{500} & \textbf{1573} & \textbf{67.2} \ci{63.0}{71.2} \\
TOWN@4096 (full) & 500 & 2565 & 71.8 \\
\textbf{IRIS@4096 (full)} & \textbf{500} & \textbf{2401} & \textbf{74.0} \ci{70.0}{77.7} \\
TOWN+@4096$^{\S}$ & 500 & 2565 & 72.4 \\
\textbf{IRIS+@4096$^{\S}$} & \textbf{500} & \textbf{2645} & \textbf{78.8} \\
\bottomrule
\multicolumn{4}{p{0.92\linewidth}}{\scriptsize All results on H800 except Nothink@512 and Nothink@1024 (A100; see Appendix~\ref{app:cross-hardware}). Nothink@512 is a development pilot; the full-set A100 diagnostic is 40.6\% (Appendix~\ref{app:math500-full}). $^{\dagger}$H800 hardware. $^{\ddagger}$Fixed non-oracle agreement gate: use SC if its top answer class has at least three votes, otherwise IRIS+. $^{\S}$Strengthened extraction: 512-token answer budget plus retry-on-fallback. All displayed rows use seed=42.} \\
\end{tabular}
\end{table}

Table~\ref{tab:compute-matched} presents the main method comparison.
On GSM8K (Table~\ref{tab:gsm8k-mrsd-full}), full-set \method{} reaches \textbf{90.90\%} at 287.5 average tokens, significantly improving over nothink@256 (+3.41\,pp; 54 wins vs.\ 9 losses; $p=6.1{\times}10^{-9}$) and over coupled \town{} (+4.93\,pp; 68 wins vs.\ 3 losses; $p=5.1{\times}10^{-17}$).
IRIS-single is slightly higher (91.43\%) and cheaper (204.4 tokens), so the GSM8K claim is not that additional rounds dominate one-round extraction; rather, full \method{} verifies that split-budget refinement yields a reliable full-set gain over non-thinking and coupled-budget routing.
On MATH-500 at full scale ($n{=}500$, H800):
\textbf{IRIS@4096 achieves 74.0\%} \ci{70.0}{77.7}---exceeding nothink@2048 (68.4\%) by +5.6\,pp, with the CI lower bound well above 68.4\%.
IRIS@2048 (67.2\%) falls slightly below nothink@2048 (68.4\%) but exceeds TOWN@2048 by \textbf{+12.2\,pp} (McNemar $p < 10^{-6}$), supporting the decoupling mechanism.
\textbf{Prefill accounting:} output tokens are the controlled budget; adding generated trace-prefill tokens raises IRIS+ from 2645 to 3931 and SC+IRIS+ from 3499 to 4122 effective tokens (Appendix~\ref{app:experiment-accounting}).
At $b{=}4096$, think@4096 finally surpasses nothink@4096 (+2.4\,pp on H800), consistent with the crossover diagnostic (Proposition~\ref{prop:crossover}); yet \textbf{IRIS@4096 still exceeds think@4096 by +3.0\,pp}, supporting the coupling-tax account.
At $b{=}2048$, the same H800 run gives nothink@2048 = 68.4\% vs.\ think@2048 = 54.8\%, so the below-crossover tax does not depend on cross-hardware comparison.
IRIS uses \emph{fewer} total tokens than TOWN at both budgets (1573 vs.\ 1590; 2401 vs.\ 2565).
With strengthened extraction, IRIS+@4096 reaches \textbf{78.8\%} on the full MATH-500 set at 2645 tokens, numerically exceeding SC@5 nothink@1024 (76.6\%, 2685 tokens) while using slightly fewer tokens; against the matched TOWN+ cascade, the paired gap is +6.4\,pp (63 wins/31 losses, $p=0.0013$).

\paragraph{Self-consistency comparison.}
Nothink SC@$k$ is a strong compute-scaling baseline, but it is complementary to split-budget extraction rather than a replacement.
On full MATH-500, a fixed non-oracle SC@5+IRIS+ agreement gate reaches \textbf{83.6\%}, improving over SC alone (76.6\%) and IRIS+ alone (78.8\%); the gate is most valuable on low-agreement SC cases, where IRIS+ recovers many failures.

\paragraph{Causal isolation: IRIS vs.\ TOWN.}
On the 106 escalated MATH-500 samples (Appendix~\ref{tab:escalated-ablation}), IRIS recovers \textbf{+25.4\,pp} over TOWN, providing strong evidence for the decoupling mechanism.

\subsection{Think Budget Ablation}
\label{sec:experiments:bthink}

Appendix~\ref{tab:bthink-ablation} shows monotonic improvement as more chains complete naturally and longer partial traces become extractable.
IRIS outperforms TOWN at every tested budget while using fewer tokens, with the narrowing gap matching the coupling-tax diagnostic in Proposition~\ref{prop:recoverable-tax}.

\subsection{Cross-Scale Validation (27B)}
\label{sec:experiments:scale}

At 27B, insufficient budgets still favor nothink because nearly all chains truncate, but adequate split budgets make decoupled extraction favorable.
The 27B IRIS--TOWN gap averages +34.5\,pp over three seeds, and a GSM8K check at $b{=}4096$ still shows a residual coupling tax, consistent with the crossover exceeding 4096 tokens for larger models.
Cross-hardware replication shows small numeric variation from bfloat16 nondeterminism, but the qualitative conclusions are invariant.

%% file: sections/conclusion_final.tex

\paragraph{Why the tax matters.}
The mechanism is truncation, but the magnitude is not a simple verbosity effect: at $b{=}256$, 98.6\% of thinking responses produce no parseable answer, yielding a 69.5\,pp loss.
The $2.1\times$ amplification from 8B to 9B/27B at $b{=}512$ and the 99.0\% natural-stop oracle further show that chain length, budget, and mode interact.

\paragraph{Mitigation and scope.}
On GSM8K, full-set \method{} improves over nothink@256 by +3.41\,pp and over coupled \town{} by +4.93\,pp; on matched MATH-500 samples, decoupling budgets recovers +25.4\,pp.
Overall, shared reasoning-and-answer budgets impose a task- and model-specific penalty below the crossover, while natural stop provides a free 99.0\% PPV routing signal.
\method{} reaches \textbf{90.9\%} on full GSM8K and \textbf{74.0\%} on full MATH-500; IRIS+ and a fixed SC@5 gate reach \textbf{78.8\%} and \textbf{83.6\%}.
The scope is fixed-budget structured reasoning; it can reduce wasted reasoning compute, but each task/model crossover should be measured before disabling CoT.

%% file: sections/appendix_final.tex

\section{Experimental Details}
\label{app:experimental_details}

\subsection{Hardware and Software}
Primary experiments are conducted on NVIDIA A100-80GB GPUs; the think budget ablation (\S\ref{sec:experiments:bthink}) and IRIS/TOWN full-scale evaluations were conducted on NVIDIA H800-80GB GPUs.
We use PyTorch 2.4.1 with CUDA 12.4 (A100) / CUDA 12.6 (H800), HuggingFace Transformers 4.51,
Python 3.10, and bfloat16 precision throughout.

\paragraph{Cross-hardware and seed disclosure.}
MATH-500 nothink@2048, think@2048, think@4096, self-consistency, IRIS, and TOWN results are from H800 hardware (seed${}=42$).
MATH-500 nothink@512 and nothink@1024 are A100 reference rows (seed${}=42$); think@1024 is an A100 diagnostic row only (seed${}=11$).
Cross-hardware replication shows ${\sim}$5\,pp variation under bfloat16
nondeterminism.
Single-mode baselines (nothink vs.\ think) are independent evaluations, not paired comparisons; all IRIS vs.\ TOWN comparisons (the primary mechanism test) are on the same H800 hardware and seed.
Greedy decoding ($\tau{=}0$) is used for all main experiments; no sampling or
nucleus filtering is applied.
A sampling robustness pilot ($\tau{=}0.7$) confirms the tax persists under stochastic decoding (Appendix~\ref{app:sampling-robustness}).

\paragraph{Answer extraction.}
Two extraction pipelines are used depending on the experiment script:
(1)~\texttt{run\_experiment.py} applies a 32-token projection pass when thinking mode exhausts its budget, giving thinking mode a free continuation bonus;
(2)~\texttt{run\_nothink\_baseline.py} uses multi-level heuristic extraction (\texttt{\textbackslash boxed\{\}}, final-answer markers, last-number fallback) without projection.
The main-text 8B GSM8K results ($n{=}1{,}319$, seed${}=42$) use pipeline~(2); all conclusions hold under both pipelines since the heuristic extractor is strictly less favorable to thinking mode than projection.

\subsection{Budget Control}
Token budgets are controlled via the \texttt{max\_new\_tokens} parameter
in HuggingFace \texttt{model.generate()}, which caps the \emph{total}
number of newly generated tokens.
For thinking mode (\texttt{enable\_thinking=True}), this budget is shared
between the reasoning trace (\texttt{<think>...</think>}) and the final
answer.
For non-thinking mode (\texttt{enable\_thinking=False}), the entire budget
is available for the answer.
This ensures a fair comparison: both modes operate under the same total
output-token budget.

\subsection{Answer Extraction}
We use a multi-level extraction pipeline:
\begin{enumerate}[nosep,leftmargin=*]
    \item Search for \texttt{\textbackslash boxed\{...\}} patterns
    \item Search for \texttt{\#\#\#\#} markers (GSM8K convention)
    \item Search for ``Final answer:'' patterns
    \item Fall back to the last number in the output
\end{enumerate}
When thinking mode exhausts its budget without producing a final answer,
two extraction strategies are used depending on the experiment script
(see ``Answer extraction'' paragraph above).
The main 8B GSM8K results ($n{=}1{,}319$, seed${}=42$) use heuristic
extraction (pipeline~2, no projection).
The heuristic extractor recovers answers from truncated outputs via last-number fallback; without \emph{any} extraction support, think@512 accuracy drops to ${<}$6\%, confirming that truncated chains rarely produce parseable answers.
Because this recovery is favorable to thinking mode, reported thinking-mode accuracies are optimistic and the measured tax is conservative.

\paragraph{Extraction heuristic analysis.}
Both extraction pipelines (heuristic and projection) give thinking mode
answer-recovery support that non-thinking mode does not need,
ensuring a fair-or-better comparison for thinking mode.
To quantify: without \emph{any} extraction heuristic (no last-number
fallback, no projection), think@512 on GSM8K drops to ${<}$6\%, since
the vast majority of truncated chains produce no parseable answer.
With heuristic extraction, accuracy recovers to 56.9\%.
Critically, this \emph{reduces} the measured thinking tax: the
``true'' tax (without answer recovery heuristics) would be even larger,
making our reported tax a conservative estimate for this evaluation protocol.

\subsection{Dataset Details}
\paragraph{GSM8K.} We use the standard test split ($n{=}1{,}319$ problems)~\citep{cobbe2021gsm8k}.
All main-text analyses use the complete set ($n{=}1{,}319$).
A 200-sample subset (seed 42) used during development is reported
in Table~\ref{tab:full-8b-200} for reference.

\paragraph{MATH-500.} We use a 500-problem subset of the MATH benchmark~\citep{hendrycks2021math},
following the split used by~\citet{lightman2023prm}.
We evaluate on the full 500-problem set ($n{=}500$) for all headline MATH-500 configurations; the nothink@512 row in Table~\ref{tab:compute-matched} is explicitly marked as a pilot baseline.
Early experiments used 40-sample subsets across data seeds \{42, 404, 505, 606, 707\} for rapid prototyping; all final results use the complete set.

\paragraph{BIG-Bench Hard (BBH).}
To validate the thinking tax beyond mathematical reasoning, we evaluate on
five subtasks from BIG-Bench Hard~\citep{suzgun2023challenging}:
\texttt{boolean\_expressions} ($n{=}250$),
\texttt{causal\_judgement} ($n{=}187$),
\texttt{date\_understanding} ($n{=}250$),
\texttt{logical\_deduction\_five\_objects} ($n{=}250$), and
\texttt{tracking\_shuffled\_objects\_three\_objects} ($n{=}250$),
totaling $n{=}1{,}187$ samples.
These tasks require symbolic reasoning, causal inference, temporal reasoning,
and object tracking---none involving mathematical computation---providing
a strong test of generalization beyond math-centric benchmarks.
Answer extraction uses exact match against the gold-standard option label.

\section{Full Data Tables}
\label{app:full-tables}

\subsection{Qwen3-8B: 200-Sample Subset}

The main text tables use the full GSM8K test set ($n{=}1{,}319$).
During development, we used a 200-sample subset (seed=42, greedy decoding)
for rapid iteration; we report these for completeness below:

\begin{table}[ht]
\centering
\small
\caption{Complete budget sweep for Qwen3-8B on GSM8K 200-sample subset.}
\label{tab:full-8b-200}
\begin{tabular}{llrrrr}
\toprule
\textbf{Mode} & \textbf{Budget} & \textbf{Accuracy} & \textbf{Avg Tokens} & \textbf{Early Stop} & \textbf{Has Final} \\
\midrule
nothink & 32 & 3.0\% & 32 & 0.0\% & 0.0\% \\
nothink & 64 & 12.0\% & 64 & 2.0\% & 0.0\% \\
nothink & 128 & 54.5\% & 111 & 43.5\% & 0.0\% \\
nothink & 256 & 89.0\% & 140 & 92.0\% & 1.5\% \\
nothink & 512 & 94.0\% & 145 & 99.5\% & 1.5\% \\
\midrule
thinking & 128 & 2.0\% & 128 & 0.0\% & 0.0\% \\
thinking & 256 & 22.0\% & 255 & 2.0\% & 0.5\% \\
thinking & 512 & 66.5\% & 442 & 47.5\% & 5.5\% \\
\bottomrule
\end{tabular}
\end{table}

This table reports the original $n{=}200$ pilot used for method development.
All headline GSM8K \method{} claims in the main text use the later full-set run in Table~\ref{tab:gsm8k-mrsd-full}.

\subsection{Qwen3-8B: Full GSM8K ($n$=1,319)}

\begin{table}[ht]
\centering
\small
\caption{Key configurations for Qwen3-8B on full GSM8K ($n{=}1{,}319$).
Early Stop here denotes $|y| < b$ (generation terminated before budget);
the stricter ``NatStop'' criterion ($<$95\% of budget) used in
Table~\ref{tab:natural-stop-oracle} yields slightly lower rates
(e.g., 37.4\% vs.\ 41.3\% at $b{=}512$).}
\label{tab:full-8b-1319}
\begin{tabular}{llrrrr}
\toprule
\textbf{Mode} & \textbf{Budget} & \textbf{Accuracy} & \textbf{Avg Tokens} & \textbf{Early Stop} & \textbf{Has Final} \\
\midrule
nothink & 128 & 50.8\% & 113 & --- & --- \\
nothink & 256 & 87.5\% & 146 & 88.8\% & 0.6\% \\
\midrule
thinking & 128 & 3.0\% & 128 & 0.0\% & --- \\
thinking & 256 & 18.0\% & 255 & 1.4\% & --- \\
thinking & 512 & 56.9\% & 460 & 37.4\% & 6.1\% \\
\bottomrule
\end{tabular}
\end{table}

\subsection{Qwen3.5-27B: Full GSM8K ($n$=1,319)}

\begin{table}[ht]
\centering
\small
\caption{Qwen3.5-27B thinking mode results on full GSM8K.
In thinking mode, the model achieves substantially lower accuracy than
8B at all budgets due to longer reasoning chains; in non-thinking mode
the 27B model reaches 95.5\% (Table~\ref{tab:model-size-scaling}).}
\label{tab:27b-full}
\begin{tabular}{lrrrr}
\toprule
\textbf{Config} & \textbf{Accuracy} & \textbf{Avg Tokens} & \textbf{Has Final} & \textbf{Projection Rate} \\
\midrule
thinking@128 & 3.6\% & 144 & 0.0\% & 100.0\% \\
thinking@256 & 7.9\% & 272 & 0.0\% & 100.0\% \\
thinking@512 & 18.4\% & 528 & 0.7\% & 99.3\% \\
\bottomrule
\end{tabular}
\end{table}

\begin{table}[ht]
\centering
\caption{\textbf{Think budget ablation on MATH-500} (Qwen3-8B, seed=42). Monotonic improvement at both pilot and full scale. 95\% Wilson CIs.}
\label{tab:bthink-ablation}
\small
\setlength{\tabcolsep}{3.5pt}
\begin{tabular}{llcccccc}
\toprule
$B_{\text{think}}$ & $n$ & \textbf{Overall (\%)} & \textbf{Stage 1} & \textbf{Stage 2} & \textbf{Stage 3} & \textbf{S3 Acc (\%)} & \textbf{Nat.\ Stop} \\
\midrule
\multicolumn{8}{l}{\textit{Full-scale ($n{=}500$)}} \\
2048 & 500 & 67.2 \ci{63.0}{71.2} & 216 & 25 & 259 & 51.0 \ci{44.9}{57.0} & 25/284 (8.8\%) \\
4096 & 500 & \textbf{74.0} \ci{70.0}{77.7} & 216 & 127 & 157 & 51.6 \ci{43.8}{59.3} & 127/284 (44.7\%) \\
\midrule
\multicolumn{2}{r}{$\Delta_{2048 \to 4096}$} & \up{6.8}$^{***}$ & --- & \up{102} & --- & --- & --- \\
\bottomrule
\multicolumn{8}{l}{\scriptsize $^{***}$McNemar $p{=}0.0004$. Pilot ($n{=}200$): 62.5\% $\to$ 73.0\% $\to$ 78.5\% at $B_{\text{think}} \in \{1024, 2048, 4096\}$.}
\end{tabular}
\end{table}

\begin{table}[ht]
\centering
\caption{\textbf{Ablation on escalated (hard) samples.} Pilot accuracy on 106 escalated MATH-500 samples from the $n{=}200$ seed-42 run that failed nothink triage. All chains truncated at $B_{\text{think}}{=}1024$.}
\label{tab:escalated-ablation}
\small
\begin{tabular}{lcc}
\toprule
\textbf{Method} & \textbf{Acc (\%)} & \textbf{$\Delta$ vs.\ TOWN} \\
\midrule
\town{} (truncated think) & 10.4 \ci{5.9}{17.6} & --- \\
IRIS (1-round decoupled) & 35.8 \ci{27.4}{45.3} & \up{25.4}$^{*}$ \\
\method{} (3-round) & 42.5 \ci{33.5}{52.0} & \up{32.1}$^{*}$ \\
\bottomrule
\multicolumn{3}{l}{\scriptsize $^{*}$McNemar $p{<}2{\times}10^{-6}$ (IRIS), $p{<}10^{-7}$ (\method{}).}
\end{tabular}
\end{table}

\subsection{Qwen3-32B: Prospective A800 Holdout}
\label{app:qwen3-32b-a800}

As a supplementary scale check, we ran a pre-registered Qwen3-32B holdout on GSM8K using greedy HF generation with the same native \texttt{enable\_thinking=True/False} interface.
The split was fixed before launch: seed 43042, shuffled-order offsets 50--169 for $b{=}512$, 200--259 for $b{=}1024$, and 300--339 for $b{=}2048$.
This is a targeted holdout rather than a full-set replacement for the 8B/27B experiments.
The result strengthens the fixed-budget diagnosis: below 2048, non-thinking is both more accurate and cheaper; at 2048, accuracy ties while thinking uses substantially more output tokens. The $b{=}2048$ row is underpowered for a directional accuracy claim ($n{=}40$, two discordant pairs) and is used only as a token-efficiency tie.
All values are from \texttt{results/a800\_prospective\_32b/summary\_a800\_prospective\_32b.json}.

\begin{table}[ht]
\centering
\small
\caption{\textbf{Qwen3-32B prospective GSM8K holdout} (A800, HF engine, seed 43042).
Rows are matched by dataset index.}
\label{tab:qwen3-32b-a800}
\begin{tabular}{rrrrrrrr}
\toprule
\textbf{Budget} & \textbf{$n$} & \textbf{Nothink Acc.} & \textbf{Think Acc.} & \textbf{$\Delta$} & \textbf{Tok. Ratio} & \textbf{Wins} & \textbf{$p$} \\
\midrule
512  & 120 & 98.3\% & 60.0\% & \up{38.3} & 2.76$\times$ & 46:0 & $2.8{\times}10^{-14}$ \\
1024 & 60  & 98.3\% & 88.3\% & \up{10.0} & 3.47$\times$ & 6:0  & 0.031 \\
2048 & 40  & 92.5\% & 92.5\% & 0.0 & 3.97$\times$ & 1:1 & 1.00 \\
\bottomrule
\end{tabular}
\end{table}

\section{DeepSeek-R1 Cross-Model Validation}
\label{app:deepseek}

We validate the truncation-waste mechanism on DeepSeek-R1-Distill-Llama-8B, a
Llama-architecture reasoning model. Unlike Qwen3, this model does not expose a
native \texttt{enable\_thinking=False} mode, so this section is not an end-to-end
thinking-vs.-non-thinking tax comparison. It instead checks whether a second
model family exhibits the same budget sensitivity when reasoning traces are cut
short. The traceable full-set evidence uses the projection-enabled evaluation
files listed below.

\paragraph{GSM8K full-scale results ($n{=}1{,}319$).}
Table~\ref{tab:deepseek-gsm8k} reports the full budget sweep for DeepSeek-R1
thinking mode on GSM8K from
\texttt{summary\_gsm8k\_DeepSeek\_R1\_Distill\_Llama\_8B\_20260326\_234956.json}.
Accuracy rises sharply as the budget increases from 256 to 512 tokens
(19.9\% $\to$ 59.6\%), while the projection rate falls from 85.2\% to 13.0\%.
This supports the same qualitative mechanism as the Qwen experiments: short
budgets leave many reasoning traces without an extractable final answer, and
additional budget recovers accuracy by reducing incomplete traces.

\begin{table}[ht]
\centering
\small
\caption{\textbf{DeepSeek-R1-Distill-Llama-8B on GSM8K} ($n{=}1{,}319$, thinking mode only, projection-enabled extraction).}
\label{tab:deepseek-gsm8k}
\begin{tabular}{rrrr}
\toprule
\textbf{Budget} & \textbf{Accuracy} & \textbf{Avg Tokens} & \textbf{Projection Rate} \\
\midrule
256  & 19.9\% & 265 & 85.2\% \\
512  & 59.6\% & 367 & 13.0\% \\
1024 & 63.8\% & 378 & 1.4\% \\
\bottomrule
\end{tabular}
\end{table}

\paragraph{MATH-500 full-scale results ($n{=}500$).}
The complete MATH-500 run shows the same monotonic budget-recovery pattern.
Because we do not have a native non-thinking DeepSeek mode, Table~\ref{tab:deepseek-math500}
reports thinking-mode results only, from
\texttt{summary\_math500\_DeepSeek\_R1\_Distill\_Llama\_8B\_20260328\_064049.json}.

\begin{table}[ht]
\centering
\small
\caption{\textbf{DeepSeek-R1-Distill-Llama-8B on MATH-500} ($n{=}500$, thinking mode only, projection-enabled extraction).}
\label{tab:deepseek-math500}
\begin{tabular}{rrrrr}
\toprule
\textbf{Budget} & \textbf{Accuracy} & \textbf{Avg Tok} & \textbf{Final Rate} & \textbf{Projection Rate} \\
\midrule
1024 & 28.4\% & 900  & 36.0\% & 65.0\% \\
2048 & 40.2\% & 1503 & 61.0\% & 39.0\% \\
4096 & 49.0\% & 2200 & 78.0\% & 22.4\% \\
\bottomrule
\end{tabular}
\end{table}

\paragraph{Key cross-architecture findings.}
DeepSeek supports the mechanism claim but not a stronger native-mode tax claim.
On GSM8K, increasing the budget from 256 to 512 tokens cuts the projection rate by
72.2\,pp and raises accuracy by 39.7\,pp. On MATH-500, larger budgets similarly
increase the final-answer rate and reduce the fraction of samples requiring a
projection pass. These results show that truncation waste is not unique to Qwen3.
They do not establish that DeepSeek has a Qwen-style end-to-end thinking tax,
because no native non-thinking counterpart is available for a controlled comparison.

\paragraph{Implication for the theory.}
The DeepSeek results validate the first condition in Proposition~\ref{prop:crossover}:
reasoning-chain length distributions can create large truncation waste under fixed
budgets. They also clarify the second condition: the full end-to-end tax can only be
measured when the model exposes an efficient alternative generation mode.

\section{Sampling Robustness: Temperature $\tau{=}0.7$}
\label{app:sampling-robustness}

All main experiments use greedy decoding ($\tau{=}0$).
To test whether the thinking tax is an artifact of greedy decoding, we run a pilot with Qwen3-8B on a random GSM8K subset ($n{=}200$, seed 42) at $\tau{=}0.7$, budget 256.

\begin{table}[ht]
\centering
\small
\caption{\textbf{Sampling robustness pilot} (Qwen3-8B, GSM8K $n{=}200$, $b{=}256$, $\tau{=}0.7$).
The thinking tax persists---and is arguably \emph{more} severe---under stochastic decoding.}
\label{tab:sampling-pilot}
\begin{tabular}{lrrr}
\toprule
\textbf{Mode} & \textbf{Accuracy} & \textbf{Nat.\ Stop} & \textbf{Avg Tokens} \\
\midrule
think@256 ($\tau{=}0.7$)   & 5.0\%  & 0.0\%  & 256 \\
nothink@256 ($\tau{=}0.7$) & 62.0\% & 55.0\% & 216 \\
\midrule
\textbf{Thinking Tax}      & \multicolumn{3}{c}{\textbf{57.0\,pp}} \\
\bottomrule
\end{tabular}
\end{table}

Under greedy decoding, the same configuration yields a tax of 75.4\,pp (18.0\% vs.\ 93.4\%).
The stochastic tax is smaller in absolute terms (57.0\,pp) because nothink accuracy also drops from 93.4\% to 62.0\% under sampling, but the core phenomenon is unchanged: thinking mode at budget 256 collapses to near-zero accuracy (5.0\%) with 0\% natural stops, while non-thinking mode retains useful accuracy (62.0\%) with 55\% of responses completing within budget.
This confirms that the thinking tax is not a greedy-decoding artifact but a fundamental consequence of budget-constrained reasoning under the \texttt{<think>} framework.

\section{\method{} Full GSM8K Details}
\label{app:town-simulation}

Our executed \method{} run uses the \textbf{full GSM8K test set}
($n{=}1{,}319$) with $B_1{=}256$, $B_{\text{think}}{=}512$,
$B_{\text{answer}}{=}128$, max rounds${=}3$, and seed~42.
The source file is
\texttt{results/mrsd\_gsm8k\_full\_s42/mrsd\_Qwen3\_8B\_gsm8k\_b1256\_bt512\_ba128\_r3\_20260506\_102716.json}.

\paragraph{Full-set results.}
\begin{itemize}[nosep,leftmargin=*]
    \item \textbf{Nothink-only}: 1{,}154/1{,}319 = \textbf{87.49\%}, 146.31 average tokens.
    \item \textbf{\town{}}: 1{,}134/1{,}319 = \textbf{85.97\%}, 203.56 average tokens.
    \item \textbf{IRIS-single}: 1{,}206/1{,}319 = \textbf{91.43\%}, 204.42 average tokens.
    \item \textbf{\method{} (3-round)}: 1{,}199/1{,}319 = \textbf{90.90\%}, 287.51 average tokens.
\end{itemize}

\paragraph{Routing and convergence.}
\begin{itemize}[nosep,leftmargin=*]
    \item \textbf{Stage 0 accepted}: 1{,}171/1{,}319 (88.8\%) samples stop
      before the $B_1{=}256$ budget; 1{,}106/1{,}171 are correct.
    \item \textbf{Routed}: 148/1{,}319 (11.2\%) samples exhaust the Stage~0 budget.
    \item \textbf{Final round distribution}: 1{,}171 samples finish at round~0,
      86 at round~2, and 62 at round~3.
    \item \textbf{Convergence}: 97.04\% of samples converge before majority fallback;
      accuracy is 92.34\% on converged samples vs.\ 43.59\% on unconverged samples.
\end{itemize}

\paragraph{Paired significance.}
\method{} improves over nothink by 45 net correct answers
(54 \method{}-only wins vs.\ 9 nothink-only losses; exact McNemar
$p=6.1{\times}10^{-9}$) and over \town{} by 65 net correct answers
(68 wins vs.\ 3 losses; $p=5.1{\times}10^{-17}$).
IRIS-single remains slightly stronger on GSM8K (12 IRIS-only wins vs.\ 5
\method{}-only wins; $p=0.14$), so our full-set GSM8K claim is a
significant gain over non-thinking and coupled-budget routing, not a
dominance claim over the one-round extraction baseline.

\section{Token Utilization Analysis}
\label{app:utilization}

\input{paper_figures/table_token_utilization}

Table~\ref{tab:token-utilization} reveals a striking pattern.
At $b{=}128$, Qwen3-8B uses 100\% of its budget---every token is
consumed, and the model is almost certainly still mid-reasoning when
generation is forcibly terminated.
At $b{=}256$, utilization remains near-perfect (99.7\%), with only 1.4\%
of samples stopping naturally.
But at $b{=}512$, utilization drops to \textbf{89.8\%}, with an average
of only 460 tokens consumed out of 512 available.

The pattern is even more dramatic for DeepSeek-R1 on GSM8K.
Full-scale numbers (Table~\ref{tab:deepseek-gsm8k}) show utilization dropping from 100\% at $b{=}256$ to just 43.7\% ($447/1{,}024$) at $b{=}1{,}024$.
The pilot run reported in Table~\ref{tab:token-utilization} ($n{=}40$) shows a consistent trend: 103\%, 73.5\%, 38.3\%.
Both runs confirm the same qualitative effect---DeepSeek-R1's chain-length distribution saturates well below $b{=}1{,}024$, so more than half the budget is wasted on padding.

\section{Difficulty Distribution and Impossible Questions}
\label{app:impossible}

We categorize GSM8K problems by the \emph{minimum} budget at which Qwen3-8B
solves them correctly in thinking mode.

\input{paper_figures/table_efficiency_frontier}

Table~\ref{tab:efficiency-frontier} reveals a four-tier difficulty distribution:
\begin{itemize}[nosep,leftmargin=*]
  \item \textbf{Easy} (11.8\%, $n{=}155$): solved at $b{\leq}128$.
  \item \textbf{Medium} (25.8\%, $n{=}340$): solved at $b{\leq}256$ but not at 128.
  \item \textbf{Hard} (30.7\%, $n{=}405$): solved only at $b{=}512$.
  \item \textbf{Impossible} (31.8\%, $n{=}419$): unsolved at \emph{any}
    tested budget ($b{\leq}512$).
\end{itemize}

An \emph{oracle router} that skips impossible questions and assigns the minimum
sufficient budget achieves \textbf{68.2\%} accuracy ($+3.0$\,pp over fixed-512)
at only \textbf{401} average tokens---a \textbf{21.7\%} savings.
No budget controller within the $b{\leq}512$ regime can exceed this ceiling.

\begin{figure}[ht]
\centering
\includegraphics[width=0.72\textwidth]{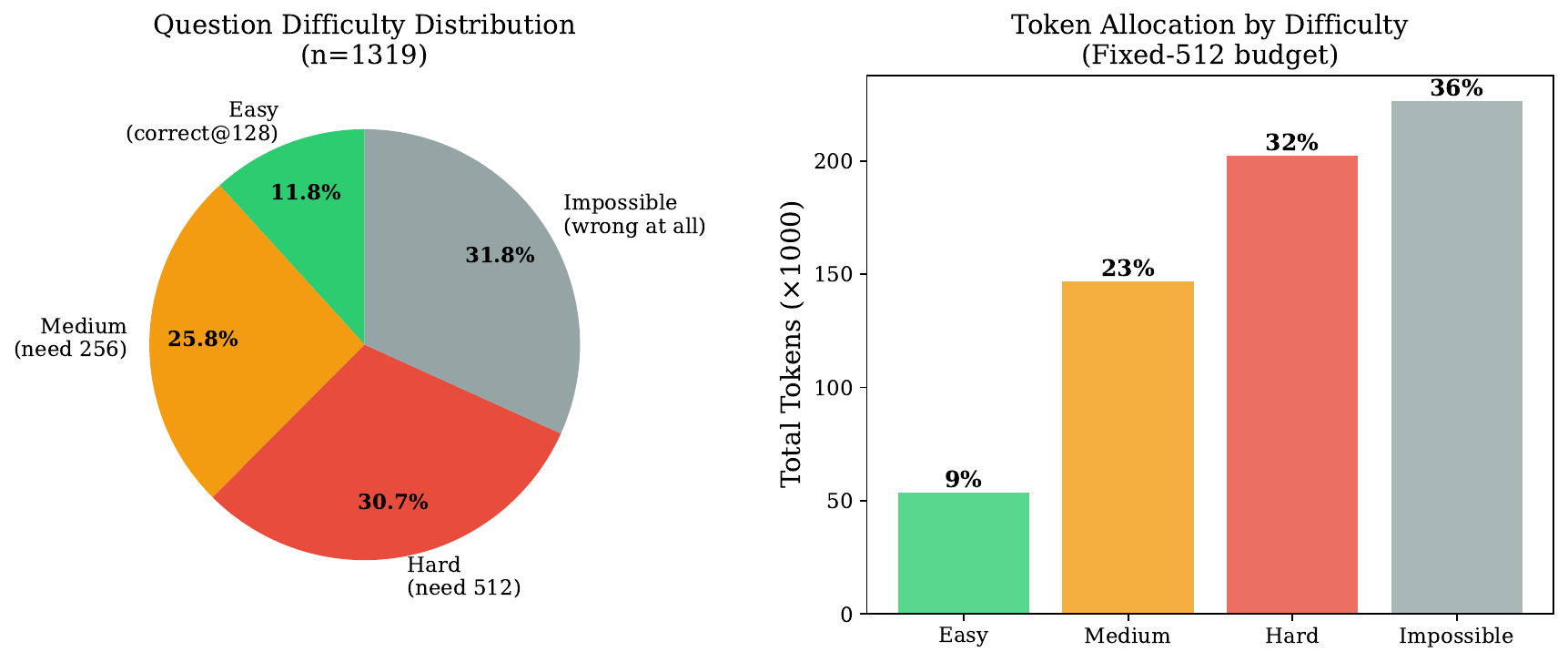}
\caption{%
  \textbf{Four-tier difficulty distribution of GSM8K problems} under Qwen3-8B thinking mode.
  31.8\% of problems are impossible at all tested budgets.
}
\label{fig:difficulty-distribution}
\end{figure}

\section{\method{} Routing Ablation}
\label{app:ablation}

\begin{table}[ht]
\centering
\caption{\textbf{Routing sensitivity to $B_1$} on full GSM8K ($n{=}1{,}319$, Qwen3-8B). All variants use $B_2{=}512$.}
\label{tab:town-ablation}
\small
\begin{tabular}{lrrrr}
\toprule
\textbf{$B_1$ (probe)} & \textbf{Early Stop} & \textbf{Routed} & \textbf{Accuracy} & \textbf{Avg Tokens} \\
\midrule
128 & 50.8\% & 49.2\% & 77.3\% & 399 \\
256 (default) & 88.8\% & 11.2\% & \textbf{90.9\%} & \textbf{199} \\
\bottomrule
\end{tabular}
\end{table}

$B_1{=}128$ yields only 50.8\% early-stop rate, aggressively routing 49.2\% to Stage~2.
$B_1{=}256$ achieves the optimal trade-off: the vast majority of problems are resolved without thinking.

\paragraph{Full $(B_1, B_2)$ sensitivity grid.}
Table~\ref{tab:town-grid} shows \method{} accuracy across all
feasible $(B_1, B_2)$ combinations, revealing when the cascade helps
vs.\ hurts.
The net recovery condition (Eq.~\ref{eq:net-recovery}) correctly predicts
the sign: configurations where $R^+ > R^-$ yield accuracy gains, while
those where the thinking fallback is too weak ($B_2$ too small) or too
few samples are routed ($B_1$ too large) show minimal or negative returns.

\begin{table}[ht]
\centering
\caption{\textbf{\method{} sensitivity to $(B_1, B_2)$} on full GSM8K
($n{=}1{,}319$, Qwen3-8B).
$R^+$: genuine recoveries (nothink wrong, think correct).
$R^-$: routing regrets (nothink correct, think wrong).
Net = $R^+ - R^-$.
\method{} provides substantial gains when $B_2$ is large enough for
thinking mode to be effective and $B_1$ routes a meaningful fraction.
Note: thinking-mode accuracy here uses standard HuggingFace generation
(without the projection pass used in some experiments),
so the $B_1{=}256$, $B_2{=}512$ row shows lower accuracy than the main
\method{} result; the qualitative patterns---$B_2$ must be large
enough, $B_1{=}256$ is optimal---are consistent across settings.}
\label{tab:town-grid}
\small
\begin{tabular}{rr|rrrrrr}
\toprule
$B_1$ & $B_2$ & \textbf{Acc} & \textbf{Avg Tok} & \textbf{Routed} & $R^+$ & $R^-$ & \textbf{Net} \\
\midrule
128 & 512 & 63.8\% & 412 & 60.6\% & 237 & 66 & +171 \\
128 & 1024 & 87.2\% & 552 & 60.6\% & 492 & 12 & +480 \\
\midrule
256 & 512 & 86.4\% & 204 & 11.2\% & 22 & 37 & $-$15 \\
256 & 1024 & \textbf{90.7\%} & \textbf{248} & 11.2\% & 53 & 11 & \textbf{+42} \\
256 & 2048 & 93.3\% & 287 & 11.2\% & 78 & 1 & +77 \\
\midrule
512 & 1024 & 93.3\% & 154 & 0.2\% & 0 & 1 & $-$1 \\
512 & 2048 & 93.5\% & 156 & 0.2\% & 1 & 0 & +1 \\
\bottomrule
\end{tabular}
\end{table}

Key patterns: (1)~$B_1{=}128$ over-routes (60.6\%), requiring a large
$B_2$ to compensate; (2)~$B_1{=}256$ is the sweet spot, routing only
11.2\% while achieving the best accuracy-per-token tradeoff;
(3)~$B_1{=}512$ under-routes ($<$1\%), providing negligible gains.
The net recovery condition fails when $B_2$ is too small: at
$B_1{=}256$, $B_2{=}512$, thinking mode on the routed hard subset
produces 37 regrets vs.\ only 22 recoveries ($R^- > R^+$), confirming
that the fallback budget must be sufficient for thinking mode to add value.

\section{MATH-500 by Difficulty Level}
\label{app:math500-level}

\begin{table}[ht]
\centering
\caption{\textbf{MATH-500 thinking accuracy by difficulty level} (Qwen3-8B).
Harder problems have longer chains and are more severely truncated.}
\label{tab:math500-level}
\small
\begin{tabular}{lrrr}
\toprule
\textbf{Level} & \textbf{Think@512} & \textbf{Think@1024} & \textbf{Think@2048} \\
\midrule
Level 1 ($n{=}43$) & 11.6\% & 44.2\% & 81.4\% \\
Level 2 ($n{=}90$) & 8.9\% & 31.1\% & 60.0\% \\
Level 3 ($n{=}105$) & 8.6\% & 21.0\% & 50.5\% \\
Level 4 ($n{=}128$) & 4.7\% & 9.4\% & 38.3\% \\
Level 5 ($n{=}134$) & 2.2\% & 6.7\% & 21.6\% \\
\bottomrule
\end{tabular}
\end{table}

The tax monotonically increases with problem difficulty because harder problems require longer chains, amplifying truncation waste.

\subsection{MATH-500: Full Data Table}
\label{app:math500-full}

\begin{table}[ht]
\centering
\small
\caption{A100 diagnostic budget sweep for Qwen3-8B on MATH-500 ($n{=}500$).
Nothink rows use seed~42; thinking rows use seed~11. Same-seed H800 main comparisons are in Table~\ref{tab:compute-matched}.
Non-thinking mode saturates at ${\sim}$64.4\% by budget 2048.
Thinking mode has \emph{not} reached parity at $b{=}2048$
(44.0\% vs.\ 64.4\%), confirming the crossover shifts further
right on harder benchmarks.}
\label{tab:math500-full}
\begin{tabular}{llrrrrr}
\toprule
\textbf{Mode} & \textbf{Budget} & \textbf{Accuracy} & \textbf{Correct} & \textbf{Avg Tokens} & \textbf{Early Stop} \\
\midrule
nothink & 256 & 16.6\% & 83 & 250 & 10.4\% \\
nothink & 512 & 40.6\% \ci{36.4}{45.0} & 203 & 429 & 42.8\% \\
nothink & 1024 & 59.8\% \ci{55.6}{64.0} & 299 & 606 & 76.2\% \\
nothink & 2048 & 64.4\% \ci{60.2}{68.6} & 322 & 763 & 87.8\% \\
\midrule
thinking & 256 & 4.2\% & 21 & 256 & 0.0\% \\
thinking & 512 & 6.2\% \ci{4.2}{8.4} & 31 & 544 & 0.0\% \\
thinking & 1024 & 18.0\% \ci{14.6}{21.4} & 90 & 1051 & --- \\
thinking & 2048 & 44.0\% \ci{39.8}{48.4} & 220 & 1978 & --- \\
\bottomrule
\end{tabular}
\end{table}

\section{Comparison with Alternative Strategies}
\label{app:comparison}

\paragraph{vs.\ \texttt{think@512} only.}
\method{} achieves substantially higher accuracy at
${\sim}$2.3$\times$ fewer tokens by avoiding thinking mode's truncation waste on the
88.8\% of problems where it is unnecessary.

\paragraph{vs.\ \texttt{nothink@256} only.}
\method{} adds a thinking fallback for the $\sim$11.2\% of hard
problems, recovering accuracy at a modest cost increase
(199 vs.\ 146 avg tokens).

\paragraph{vs.\ Think-then-stop (early termination).}
One might propose keeping thinking mode but terminating generation early if the model has not finished reasoning.
However, this is equivalent to our ``truncated'' category: at $b{=}512$, only 37.4\% of 8B samples produce a complete chain, and forcing early termination discards both the reasoning progress and any partial answer.
Unlike non-thinking mode, which can produce a useful answer in far fewer tokens because it does not waste budget on a reasoning chain, early-terminated thinking chains rarely contain extractable answers (without any answer-recovery heuristic, accuracy drops to ${<}$6\%).
\method{} effectively implements the optimal version of this idea: use non-thinking mode first (which naturally stops early on easy problems), and only invoke thinking for genuinely hard cases.

\paragraph{vs.\ Self-consistency~\cite{wang2023selfconsistency}.}
SC@$K$ generates $K$ full-budget reasoning traces and majority-votes,
consuming $K \times B$ tokens unconditionally.
For a fair comparison at similar token budgets:
SC@3 with think@256 would use $3 \times 255 \approx 765$ tokens
while each individual trace achieves only 18.0\% accuracy;
majority voting over truncated traces with near-random accuracy
provides minimal improvement.
By contrast, \method{} achieves 90.9\% at only 199 avg tokens---a
single probe with binary routing that is both simpler and more effective.
SC@$K$ is more appropriate when individual traces are already accurate
(i.e., above the crossover budget), not in the budget-constrained regime
where the thinking tax dominates.

\paragraph{vs.\ Nothink SC@$K$.}
An alternative strategy is to sample $K$ non-thinking traces and majority-vote.
SC@3 with nothink@256 would consume $3 \times 146 = 438$ tokens---more than
\method{}'s 199 avg tokens.
While nothink SC@$K$ may improve over single-trace nothink@256 (87.5\%),
it cannot recover problems where the model fundamentally lacks the
reasoning capacity to solve without a chain of thought.
\method{}'s advantage is qualitative: Stage~2 thinking mode provides a
\emph{different reasoning strategy}, not just repeated attempts at the same one.

\paragraph{SC+IRIS paired validation.}
To test whether stochastic non-thinking samples and split-budget reasoning make
distinct errors, we evaluate non-oracle combination rules on paired MATH-500
outputs. The two-shard validation set contains $n{=}220$ examples from
\texttt{results/iris\_sc\_combo\_math500\_aggregate/paired\_validation\_n220.json}.
The \texttt{combo\_ge3} rule uses SC@5 when its top equivalence class has at
least three votes and otherwise falls back to IRIS@4096. This rule reaches
85.5\% accuracy at 3423 average tokens, compared with 80.9\% for SC@5 and
83.2\% for IRIS on the same examples. The paired gain over SC is significant
(11 wins/1 loss, exact McNemar $p{=}0.0063$), while the gain over IRIS is
positive but not significant (11 wins/6 losses, $p{=}0.33$). A higher-cost
weighted-vote rule reaches 86.4\%, reinforcing the conclusion that SC and IRIS
are complementary rather than redundant.

\paragraph{Full-set SC+IRIS+ combination.}
Applying the same fixed \texttt{combo\_ge3} agreement gate to the completed
full-set SC@5 and IRIS+ outputs gives a stronger non-oracle result on all
500 MATH-500 examples. The source file is
\texttt{results/iris\_sc\_combo\_math500\_full\_n500/posthoc\_full\_sc\_iris\_combo.json},
generated by \texttt{scripts/analyze\_full\_sc\_iris\_combo.py} from
\texttt{results/sc\_math500\_b1024\_k5\_h800.json} and
\texttt{results/bugfix\_8b\_math500\_n500/checkpoint\_iris\_500.json}.
SC@5 alone obtains 383/500 = 76.6\%; IRIS+ obtains 394/500 = 78.8\%.
The fixed gate obtains 418/500 = 83.6\% at 3499 average tokens, with
39 wins/4 losses versus SC@5 (exact McNemar $p=3.1{\times}10^{-8}$) and
36 wins/12 losses versus IRIS+ ($p=7.2{\times}10^{-4}$).
The SC-confident partition ($\geq$3 agreeing votes, $n{=}415$) has SC accuracy
88.9\% and IRIS+ accuracy 83.1\%; the low-confidence partition ($n{=}85$) has
SC accuracy 16.5\% and IRIS+ accuracy 57.6\%.
Split-half robustness uses the same fixed threshold without retuning: first 250
examples reach 82.4\% (vs.\ SC 74.8\%, IRIS+ 77.6\%; paired $p=1.6{\times}10^{-4}$
vs.\ SC and $p=0.029$ vs.\ IRIS+), and the second 250 reach 84.8\%
(vs.\ SC 78.4\%, IRIS+ 80.0\%; $p=1.4{\times}10^{-4}$ and $p=0.017$).
A calibration/test split gives the same conclusion: selecting the threshold on
idx 0--99 chooses $\geq$3 votes, and applying it unchanged to idx 100--499
gives 337/400 = 84.3\% at 3447 output tokens, compared with SC 77.8\% and
IRIS+ 80.3\%; paired tests give 30 wins/4 losses vs.\ SC ($p=6.2{\times}10^{-6}$)
and 25 wins/9 losses vs.\ IRIS+ ($p=0.009$).
For trace-prefill accounting, IRIS+ averages 2645 generated tokens plus 1286
generated trace-prefill tokens (3931 total), while SC+IRIS+ averages 3499
generated tokens plus 623 trace-prefill tokens (4122 total). We report generated
output tokens as the controlled budget and disclose trace-prefill overhead because
prefill is not free in latency-sensitive deployments.

\paragraph{Strengthened extraction on full MATH-500.}
We also evaluate a strengthened extraction instantiation, denoted IRIS+, that
uses the same triage and thinking budgets as IRIS@4096 but increases the
dedicated answer budget to 512 tokens and retries extraction when the first
answer pass falls back to a weak parse. The full-set result is stored in
\texttt{results/bugfix\_8b\_math500\_n500/checkpoint\_iris\_500.json}, with the
matched TOWN+ baseline in \texttt{results/bugfix\_8b\_math500\_n500/town\_only\_20260428\_155832.json}.
On all 500 MATH-500 examples, IRIS+ obtains 394/500 = 78.8\% at 2645 average
tokens, while TOWN+ obtains 362/500 = 72.4\% at 2565 tokens. The paired gap is
63 IRIS-only wins versus 31 TOWN-only wins (exact McNemar $p=0.0013$). The same
sample order matches the H800 SC@5 file exactly by index and gold answer; IRIS+
also exceeds SC@5 nothink@1024 (76.6\%, 2685 tokens), with a positive but
non-significant paired gap (51 wins/40 losses, $p=0.29$).

\paragraph{vs.\ Oracle routing.}
An oracle that knows each problem's difficulty \emph{a priori}
achieves 68.2\% accuracy at 401 average tokens.
\method{}'s early-stop signal approximates this oracle using a training-free heuristic.

\section{Thinking Efficiency Frontier}
\label{app:frontier}

We categorize GSM8K problems by the minimum thinking budget required
for correct answers (see Table~\ref{tab:efficiency-frontier} in the
main text).
The 31.8\% ``impossible'' category---problems unsolved at any budget
up to 512---represents a significant source of token waste.

\paragraph{Oracle analysis.}
An oracle that perfectly identifies problem difficulty and assigns
the minimum sufficient budget to solvable questions achieves:
\begin{itemize}[nosep,leftmargin=*]
    \item \textbf{68.2\%} accuracy ($+3.0$pp over fixed-512)
    \item \textbf{401} average tokens ($-21.7$\% vs.\ fixed-512)
\end{itemize}
The accuracy gain comes from avoiding ``overthinking'' cases where
extended reasoning causes the model to revise a correct intermediate
answer.

\section{Format-Adjusted Fairness Experiments}
\label{app:fairness}

To address the concern that comparing \texttt{think@$b$} and \texttt{nothink@$b$} at the same total token budget is ``unfair'' to thinking mode, we design controlled experiments that give thinking mode various advantages.

\begin{table}[ht]
\centering
\caption{\textbf{Format-adjusted fairness experiments} on full GSM8K ($n{=}1{,}319$, Qwen3-8B).
Even with generous answer-buffer guarantees and 2$\times$ the token budget, thinking mode cannot match non-thinking.
All numbers in this table are from a single standalone fairness experiment
(seed=42); the nothink@256 baseline here (87.0\%) differs slightly from
the main evaluation (87.5\%, seed=42) due to run-level variance.}
\label{tab:fairness}
\small
\begin{tabular}{lrrrr}
\toprule
\textbf{Configuration} & \textbf{Accuracy} & \textbf{Avg Tokens} & \textbf{Budget Hit} & \textbf{vs nothink@256} \\
\midrule
\texttt{nothink@256} (baseline) & 87.0\% & 147 & 11.5\% & --- \\
\texttt{think@256} & 18.0\% & 255 & 98.6\% & $-$69.0pp \\
\texttt{think@512\_generous} & 56.0\% & 460 & 62.6\% & $-$31.0pp \\
\texttt{think@256+nothink@256} & 67.6\% & 263 & --- & $-$19.4pp \\
\texttt{think@512} & 56.0\% & 460 & 62.6\% & $-$31.0pp \\
\bottomrule
\end{tabular}
\end{table}

\paragraph{Configurations.}
\begin{itemize}[nosep,leftmargin=*]
    \item \textbf{\texttt{think@512\_generous}}: Allocates 512 total tokens, but guarantees the last 256 tokens are reserved for the answer.  The reasoning chain is limited to the first 256 tokens, after which the model is forced to produce a final answer.  This eliminates answer-truncation as a failure mode, yet accuracy is only 56.0\%---31.0\,pp below \texttt{nothink@256} at double the budget.
    \item \textbf{\texttt{think@256+nothink@256} (two-pass)}: First generates a 256-token reasoning chain in thinking mode, then feeds that chain as context to a non-thinking completion with a 256-token answer budget.  This hybrid achieves 67.6\%---better than pure thinking but still 19.4\,pp below non-thinking alone.
    \item \textbf{\texttt{think@512}}: Standard thinking mode with 512 total tokens and no answer-buffer guarantee.  Achieves only 56.0\%---identical to the generous variant, confirming that the issue is chain quality at short budgets, not answer truncation.
\end{itemize}

These results confirm that the thinking tax is not merely about answer truncation: even when the answer is \emph{guaranteed} to fit, the reasoning chain's quality at short budgets is insufficient to improve over direct answering.

\section{Additional Figures}
\label{app:model-size-figure}

\begin{figure}[ht]
\centering
\includegraphics[width=0.72\textwidth]{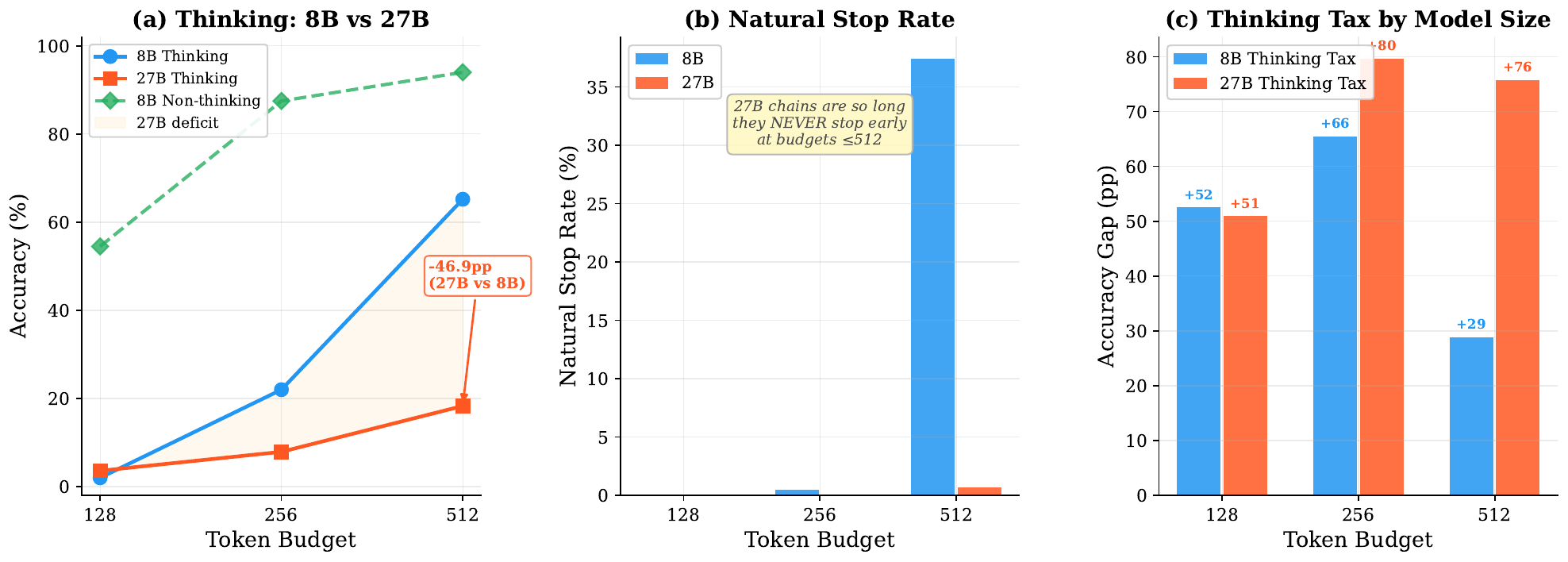}
\caption{%
  \textbf{The thinking tax worsens with model size.}
  At $b{=}512$, thinking-mode accuracy collapses with model size
  (8B: 56.9\%, 9B: 15.5\%, 27B: 18.4\%)---the 9B and 27B taxes are both ${\sim}2.1\times$ larger at this budget (vs.\ uniform nothink: 93.1\%, 93.2\%, 95.5\%).
  The 8B non-thinking baseline (green dashed) dominates all thinking configurations at every budget $\leq$512.
}
\label{fig:thinking-tax-model-size}
\end{figure}

\input{sections/fig_town_pipeline}


\section{Full Theoretical Analysis}
\label{app:theory-full}

The main text presents the core decomposition (Eqs.~\ref{eq:acc-decomp}--\ref{eq:crossover}) with derivations for the accuracy decomposition and crossover budget.
Here we provide the full formal definitions, additional propositions, and remaining proofs.

\subsection{Formal Definitions}

\begin{definition}[Thinking chain length]
For a model $\mathcal{M}$ and question $q$, let $L(q) \in \mathbb{N}$
denote the \emph{natural chain length}---the total number of output tokens
the model would generate in thinking mode if unconstrained
(i.e., $b \to \infty$).
Let $F_L(t) \triangleq \Pr(L \le t)$ denote the CDF of $L$ over the
question distribution $q \sim \mathcal{Q}$.
\end{definition}

\begin{definition}[Truncation rate]
Given token budget $b$, the \emph{truncation rate} is
$\rho(b) \triangleq 1 - F_L(b) = \Pr(L > b)$.
\end{definition}

\begin{assumption}[Binary outcome structure (restated)]
\label{asm:binary-restated}
When the chain completes naturally ($L \le b$), the model produces a
well-formed answer with accuracy $\alpha_c(b) \triangleq
\Pr(\text{correct} \mid L \le b) \in (0,1]$.
When truncated ($L > b$), the residual accuracy is
$\alpha_t(b) \triangleq \Pr(\text{correct} \mid L > b) \ll \alpha_c(b)$.
Both depend on $b$ through the conditioning event; we suppress the
argument when the dependence is negligible.
\end{assumption}

\noindent
At $b{=}512$ on GSM8K ($n{=}1{,}319$), $\alpha_c = 99.0\%$ among natural-stop samples vs.\
$\alpha_t = 31.8\%$ among truncated samples (Finding~2).
At lower budgets, $\alpha_t$ approaches zero.

\begin{remark}[Budget-dependent $\alpha_t$]
In practice, $\alpha_t$ varies with $b$: at $b{=}128$, nearly all
truncated samples lack any answer ($\alpha_t \approx 0$);
at $b{=}512$, answer-extraction heuristics and projection passes
recover partial answers ($\alpha_t \approx 31.8\%$).
The $\alpha_t = 0$ simplification serves as a tight first-order
approximation at low budgets.
\end{remark}

\subsection{Budget Multiplier}

\begin{corollary}[Budget multiplier (restated)]
\label{cor:multiplier-restated}
The budget multiplier $\gamma \triangleq b^*/b_{\mathrm{sat}}$ quantifies
the token overhead for thinking to become viable.
On GSM8K: $\gamma \approx 2048/512 = 4\times$.
On MATH-500: $\gamma > 2048/1024 = 2\times$, since thinking has not yet caught non-thinking at $b{=}2048$, confirming that harder tasks demand larger multipliers.
\end{corollary}

\subsection{Natural Stop as a Confidence Oracle}

\begin{proposition}[Oracle precision (restated)]
\label{prop:oracle-restated}
Under Assumption~\ref{asm:binary-restated}, the natural-stop event
$\mathcal{S}(b) \triangleq \{L(q) \le b\}$ has positive predictive value
$\mathrm{PPV}(\mathcal{S}) = \Pr(\mathrm{correct} \mid L \le b) = \alpha_c(b)$.
In practice, $|y| < b$ serves as an empirical proxy for $\{L \le b\}$.
Non-thinking PPV is empirically high (94.4\% at $b{=}256$ on GSM8K)
but does not follow from Assumption~\ref{asm:binary-restated} alone.
\end{proposition}

\noindent
By Hoeffding's inequality, with $n{=}475$ samples from the
$|y| < 0.95\,b$ proxy subset (8B, $b{=}512$)
and observed proxy PPV $\hat{\alpha} = 0.963$:
the 95\% lower confidence bound on this proxy PPV is $\ge 0.907$.

\begin{remark}[Information-theoretic interpretation]
Early stopping implies the model's conditional entropy
$H(A \mid y_{1:t^*}) \approx 0$ at $t^* \ll b$---the model is
near-certain of its answer.
Unlike logit-based confidence, this signal is purely behavioral.
\end{remark}

\subsection{Inverse Scaling with Model Size}

\begin{proposition}[Inverse scaling of the thinking tax (restated)]
\label{prop:inverse-scaling-restated}
Let $L_M$ denote the chain-length distribution for a model of size $M$.
If (i) $L_{M_2}$ stochastically dominates $L_{M_1}$ for $M_2 > M_1$
(i.e., $F_{L_{M_2}}(b) \le F_{L_{M_1}}(b)$ for all $b$),
(ii) non-thinking accuracy $\mathrm{Acc}_{\mathrm{nt}}(b)$ is
approximately size-invariant at moderate budgets,
and (iii) the thinking-mode accuracy of completed chains satisfies
$F_{L_{M_2}}(b) \cdot \alpha_c(M_2, b) \le F_{L_{M_1}}(b) \cdot \alpha_c(M_1, b)$,
then for any fixed budget $b$ in the truncation-dominated regime
($\alpha_t \approx 0$):
$\mathrm{Tax}(M_2, b) \ge \mathrm{Tax}(M_1, b)$.
\end{proposition}

\begin{proof}
From Eq.~\ref{eq:tax} with $\alpha_t \approx 0$:
$\mathrm{Tax}(M, b) \approx \mathrm{Acc}_{\mathrm{nt}}(b) -
F_{L_M}(b) \cdot \alpha_c(M, b)$.
With $\mathrm{Acc}_{\mathrm{nt}}$ approximately constant across $M$
and condition~(iii), the second term is smaller for the larger model,
increasing the tax.
Condition~(iii) is implied by stochastic dominance alone when
$\alpha_c$ is size-invariant, but is verified empirically in the
general case: at $b{=}512$, $F_{L_{\text{8B}}} \alpha_c = 0.374 \times 0.990 = 0.370$
vs.\ $F_{L_{\text{27B}}} \alpha_c = 0.007 \times 0.963 = 0.007$.
\end{proof}

\noindent
\textbf{Empirical verification.}
At $b{=}512$: 8B natural-stop rate = 37.4\%, 27B = 0.7\%.
The thinking tax for 27B is $95.5\% - 18.4\% = 77.1$\,pp vs.\
8B's gap of $93.1\% - 56.9\% = 36.2$\,pp at budget 512---a ${\sim}2.1\times$ observed ratio at this budget.

\begin{figure}[h]
\centering
\includegraphics[width=0.7\textwidth]{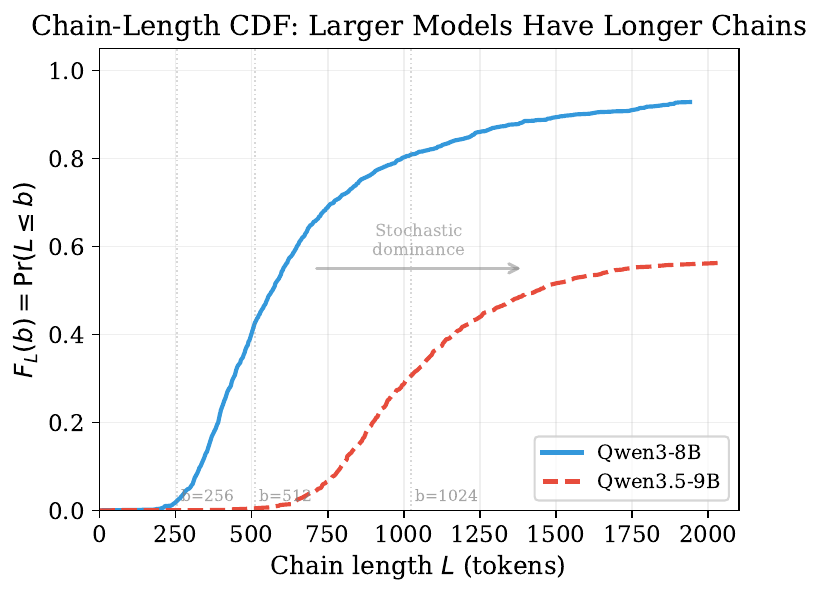}
\caption{\textbf{Chain-length CDF $F_L(b)$ for Qwen3-8B vs.\ Qwen3.5-9B on GSM8K (think@2048, $n{=}1{,}319$).}
The 9B CDF shifts right (stochastic dominance): median chain length increases from 540 tokens (8B) to 993 tokens (9B), with natural-stop rates of 92.8\% vs.\ 56.3\%.
At any fixed budget $b$, $F_{L_{\text{9B}}}(b) \le F_{L_{\text{8B}}}(b)$, meaning more 9B chains are truncated---consistent with the observed $2.1\times$ tax ratio at $b{=}512$ predicted by Proposition~\ref{prop:inverse-scaling-restated}.
Kaplan-Meier estimation accounts for right-censored chains that hit the budget ceiling.}
\label{fig:chain-length-cdf}
\end{figure}

\begin{corollary}[Crossover budget grows with model size (restated)]
\label{cor:crossover-scaling-restated}
If $\bar{\alpha}_{\mathrm{nt}}$ and $\alpha_c$ are approximately
model-invariant and $L_{M_2}$ stochastically dominates $L_{M_1}$,
then $b^*(M) \approx F_{L_M}^{-1}(\bar{\alpha}_{\mathrm{nt}}/\alpha_c)$
shifts right with $M$:
$M_2 > M_1 \implies b^*(M_2) \ge b^*(M_1)$.
Larger models require proportionally larger budgets for thinking to
become cost-effective.
\end{corollary}

\section{Natural Stop Oracle: Full Analysis}
\label{app:natural-stop-details}

Section~\ref{sec:finding:natural-stop} summarizes the natural-stop oracle (see Table~\ref{tab:natural-stop-oracle} in the main text).
Here we present the complete analysis.

At $b{=}512$ with Qwen3-8B, 37.4\% of samples terminate before consuming
95\% of their token budget (Table~\ref{tab:natural-stop-oracle}).
Among these, accuracy is \textbf{99.0\%}; among truncated samples, only 31.8\%---a 67.2\,pp gap.
A stricter definition requiring an explicit ``Final answer'' marker selects only 6.1\% but achieves 93.8\%.
The effect generalizes to DeepSeek-R1-8B (Table~\ref{tab:natural-stop-oracle}).

\section{Routing Baseline Comparison}
\label{app:routing-baselines}

Table~\ref{tab:routing_baselines} compares \method{} against budget-matched routing baselines, all routing exactly 11.2\% of samples to \texttt{think@512}.
\textbf{Random routing} yields only 85.0\% (45.6 regrets overwhelm 12.7 recoveries).
\textbf{Inverse token-length routing} performs even worse (83.5\%).
\method{}'s signal is equivalent to a \textbf{token-length router}---optimal within the class of budget-utilization-based routers.
The \textbf{oracle} upper bound (95.8\%) leaves a 4.9\,pp gap reflecting routing regret and missed recoveries.

\begin{table}[ht]
\centering
\caption{\textbf{Routing baseline comparison} on GSM8K (Qwen3-8B, $n{=}1{,}319$).
All budget-matched routers use the same 11.2\% routing rate as \method{}.
Think@512 row reports values from the routing analysis run (seed~11,
accuracy 65.2\%); the main-paper seed~42 run gives 56.9\% at similar
avg tokens, with identical qualitative conclusions.}
\label{tab:routing_baselines}
\small
\begin{tabular}{@{}lcccccc@{}}
\toprule
\textbf{Method} & \textbf{Acc.\,(\%)} & \textbf{Avg Tok.} & \textbf{Route\,\%} & \textbf{Recov.} & \textbf{Regret} & \textbf{Net\,$\Delta$} \\
\midrule
NoThink@256 & 87.5 & 146 & 0.0 & 0 & 0 & 0 \\
Think@512 & 65.2 & 477 & 100.0 & 109 & 403 & $-$294 \\
\midrule
Random & $85.0{\pm}0.5$ & $200{\pm}1$ & 11.2 & 12.7 & 45.6 & $-$32.9 \\
Inverse Token-Length & 83.5 & 200 & 11.2 & 1 & 54 & $-$53 \\
Token-Length & 90.9 & 199 & 11.2 & 64 & 19 & 45 \\
Oracle & 95.8 & 182 & 8.3 & 109 & 0 & 109 \\
\midrule
\textbf{\method{} (Ours)} & \textbf{90.9} & \textbf{199} & 11.2 & 64 & 19 & \textbf{45} \\
\bottomrule
\end{tabular}
\end{table}

\section{Full MATH-500 Results}
\label{app:math500-full-results}

Table~\ref{tab:math500-tax-full} below shows the A100 MATH-500 thinking tax results.
Here we include the extended version with the additional $b{=}128$ data point (see also Table~\ref{tab:math500-full} for per-mode details).

\begin{table}[ht]
\centering
\caption{\textbf{Extended A100 MATH-500 thinking tax} ($n{=}500$, Qwen3-8B).
Includes low-budget results. At budget 256, thinking achieves just 4.2\% with 0\% early stop rate---every response is truncated.}
\label{tab:math500-tax-full}
\small
\begin{tabular}{rrrr}
\toprule
\textbf{Budget} & \textbf{Nothink Acc} & \textbf{Think Acc} & \textbf{Gap (NT$-$T)} \\
\midrule
128 & 4.6\% & --- & --- \\
256 & 16.6\% & 4.2\% & $+$12.4pp \\
512 & 40.6\% \ci{36.4}{45.0} & 6.2\% \ci{4.2}{8.4} & $+$34.4pp \\
1024 & 59.8\% \ci{55.6}{64.0} & 18.0\% \ci{14.6}{21.4} & $+$41.8pp \\
2048 & 64.4\% \ci{60.2}{68.6} & 44.0\% \ci{39.8}{48.4} & $+$20.4pp \\
\bottomrule
\end{tabular}
\end{table}

The thinking tax is confirmed and amplified on MATH-500: at budget 1024, the gap is \textbf{41.8\,pp} \ci{36.4}{47.0} (paired bootstrap)---even larger than GSM8K's 28.2\,pp at the same budget.
The A100 crossover has not been reached at 2048 tokens (20.4\,pp gap), confirming that harder benchmarks push the crossover further right; the same-H800 reference at 2048 remains positive but smaller (68.4\% nothink vs.\ 54.8\% thinking).


\section{Theory Verification}
\label{app:theory-verification}

Table~\ref{tab:theory-verification} validates the truncation-waste decomposition
(Eq.~\ref{eq:acc-decomp}) across three dimensions: budgets, model scales, and benchmarks.

\begin{table}[ht]
\centering
\caption{\textbf{Theory verification: predicted vs.\ observed accuracy.}
The framework $\mathrm{Acc}(b) = F_L(b)\cdot\alpha_c + (1{-}F_L(b))\cdot\alpha_t$
is checked as an accounting identity at $b{\in}\{256,512\}$ for 8B
where all components are directly measurable.
The crossover budget $b^*$ corresponds to the $\bar{\alpha}_{\mathrm{nt}}/\alpha_c$ quantile of chain lengths
(Eq.~\ref{eq:crossover-exact}): at $b{=}2048$, $F_L \approx 0.93$ (93\% of chains complete), yielding
$0.93 \times 0.990 \approx 92.1\%$---close to the observed 93.1\%, confirming the crossover occurs near $b{=}2048$.
Cross-scale: the framework accounts for the 27B collapse
($F_L(512)_{27\text{B}} = 0.007$ vs.\ $0.374$ for 8B).}
\label{tab:theory-verification}
\small
\begin{tabular}{lccccc}
\toprule
\textbf{Config} & $F_L(b)$ & $\alpha_c$ & $\alpha_t$ & \textbf{Predicted} & \textbf{Observed} \\
\midrule
\multicolumn{6}{l}{\textit{Directly measured (8B, GSM8K)}} \\
8B, $b{=}256$ & 0.014 & 1.000 & 0.168 & 18.0\% & 18.0\% \\
8B, $b{=}512$ & 0.374 & 0.990 & 0.318 & 56.9\% & 56.9\% \\
\midrule
\multicolumn{6}{l}{\textit{Crossover prediction ($F_L \approx \bar{\alpha}_{\mathrm{nt}}/\alpha_c$)}} \\
8B, $b{=}2048$ & ${\sim}$0.93 & 0.990 & --- & ${\sim}$93\% & 93.1\% \\
\midrule
\multicolumn{6}{l}{\textit{Cross-scale prediction}} \\
9B, $b{=}512$ & ${\sim}$0.01 & --- & --- & $<$5\% & 15.5\% \\
27B, $b{=}512$ & 0.007 & --- & 0.178 & 18.4\% & 18.4\% \\
\midrule
\multicolumn{6}{l}{\textit{Cross-benchmark (MATH-500, same model)}} \\
8B, $b{=}512$ & ${\sim}$0.00 & --- & 0.062 & 6.2\% & 6.2\% \\
8B, $b{=}1024$ & 0.002 & 1.000 & 0.178 & 18.0\% & 18.0\% \\
8B, $b{=}2048$ & 0.178 & 0.787 & 0.365 & 44.0\% & 44.0\% \\
\bottomrule
\end{tabular}
\end{table}

The framework's predictions are \textbf{consistent with observations} wherever all components ($F_L$, $\alpha_c$, $\alpha_t$)
are directly measurable (8B at $b{\in}\{256, 512\}$); this is expected since the decomposition is an accounting identity when the binary outcome assumption holds.
The crossover prediction is confirmed: at $b{=}2048$,
$F_L \approx 0.93$ (i.e., 93\% of chains complete within budget), yielding a predicted accuracy of
$0.93 \times 0.990 \approx 92.1\%$, close to the observed 93.1\%---confirming that the crossover occurs near $b{=}2048$ as predicted.
The cross-scale predictions correctly capture the qualitative ordering
(27B $\ll$ 9B $\ll$ 8B at same budget) and the mechanism (longer chains $\rightarrow$ smaller $F_L$
$\rightarrow$ more truncation waste).
At 9B, observed accuracy (15.5\%) slightly exceeds the $\alpha_t{=}0$ prediction ($<$5\%) because
the projection pass (Appendix~\ref{app:experimental_details}) recovers partial answers
from near-complete chains.

The key \emph{predictive} insight is that the crossover budget $b^*$ is
approximately determined by the chain-length distribution, non-thinking
saturation accuracy, and completed-chain accuracy $\alpha_c$
(Eq.~\ref{eq:crossover-exact}; the heuristic Eq.~\ref{eq:crossover}
further assumes $\alpha_t \approx 0$ and stable $\alpha_c$):
for any model where these quantities are known, the crossover can be estimated
\emph{without} running the expensive budget sweep that traditionally determines it.
For example, estimating $F_L$ from a small pilot run at a single budget can
screen whether thinking mode is likely to be beneficial at a target deployment budget.
In \texttt{results/pilot\_cdf\_sufficiency/pilot\_cdf\_sufficiency\_results.json},
20 random GSM8K calibration pilots with $n_{\mathrm{cal}}{=}50$ achieve
3.48\,pp average RMSE (0.69\,pp std.) across the evaluated budget sweep.

\paragraph{Held-out prediction tests and limits.}
To assess the decomposition's \emph{out-of-sample} predictive power (beyond the in-sample accounting identity),
we perform two held-out tests.
On BBH, estimating $\alpha_c{=}0.950$ and $\alpha_t{=}0.275$ from $b{=}512$ data
predicts $b{=}2048$ accuracy within \textbf{0.8\,pp} and $b{=}1024$ within 4.3\,pp
(main text, \S\ref{sec:theory}).
The larger $b{=}1024$ error arises because $\alpha_t$ increases with budget as more truncated
samples produce partial answers that happen to be correct.

This reveals a genuine limitation: the decomposition assumes $\alpha_c$ and $\alpha_t$ are
approximately constant across budgets, which holds for moderate extrapolation
(BBH $b{=}512{\to}2048$: $\alpha_t$ shifts from 0.275 to 0.283, yielding 0.8\,pp error)
but breaks down when the budget gap is large relative to chain-length variance
(MATH-500 $b{=}1024{\to}2048$: $\alpha_t$ shifts from 0.178 to 0.365, yielding ${\sim}$12\,pp error).
The practical implication is that the crossover heuristic ($\bar{\alpha}_{\mathrm{nt}}/\alpha_c$ quantile of chain lengths)
remains reliable---it depends only on $F_L$ and $\bar{\alpha}_{\mathrm{nt}}$, not on $\alpha_t$---but
point predictions of thinking-mode accuracy at specific budgets degrade when $\alpha_t$
varies substantially across the prediction range.

\section{\method{} Interpolation-Dominance Proof}
\label{app:town-proof}

\begin{proof}[Full proof of Theorem~\ref{thm:mrsd-pareto}]
We prove the two claims (accuracy dominance and cost dominance)
separately.

\paragraph{Accuracy bound.}
Let $p = \Pr(\mathcal{S}_{\mathrm{nt}}(B_1))$ be the early-stop rate.
Decompose nothink accuracy:
$$\mathrm{Acc}_{\mathrm{nt}}(B_1) = p \cdot \alpha_c^{\mathrm{nt}}
  + (1-p) \cdot \alpha_{\mathrm{nt}}^{\mathrm{hard}},$$
where $\alpha_{\mathrm{nt}}^{\mathrm{hard}} \triangleq
\Pr(\mathrm{correct} \mid \neg\mathcal{S}, \mathrm{nothink})$.
\textsc{Mrsd} replaces the second term with
$\alpha_{\mathrm{extract}}^{\mathrm{hard}}$
(the accuracy of the complete escalated branch, including
all refinement rounds and majority vote, on queries that
exhaust the Stage~0 budget):
$$\mathrm{Acc}_{\textsc{Mrsd}} = p \cdot \alpha_c^{\mathrm{nt}}
  + (1-p) \cdot \alpha_{\mathrm{extract}}^{\mathrm{hard}}.$$
Subtracting:
\begin{equation}
\mathrm{Acc}_{\textsc{Mrsd}} - \mathrm{Acc}_{\mathrm{nt}}(B_1)
= (1-p)\bigl[\alpha_{\mathrm{extract}}^{\mathrm{hard}}
  - \alpha_{\mathrm{nt}}^{\mathrm{hard}}\bigr].
\label{eq:town-gain-decomp}
\end{equation}
The net recovery condition~\eqref{eq:net-recovery} ensures the bracket
is positive, so $\mathrm{Acc}_{\textsc{Mrsd}} > \mathrm{Acc}_{\mathrm{nt}}(B_1)$.

\paragraph{Cost bound.}
By Eq.~\eqref{eq:mrsd-cost},
$\mathbb{E}[T] \le p \bar{t}_1 + (1-p)(B_1 + \bar{K}(B_r + B_a)),$
where $\bar{t}_1 \le B_1$ and $\bar{K} \le K$.
The maximum-budget thinking baseline uses $B_{\max} \triangleq B_1 + K(B_r + B_a)$
tokens per query.
Since $\bar{t}_1 < B_1$ when early-stop queries exist
(they terminate before exhausting $B_1$):
$\mathbb{E}[T] \le p \cdot B_1 + (1-p) \cdot B_{\max} < B_{\max},$
where the last inequality holds because $0 < p < 1$ and $B_1 < B_{\max}$.
Thus $\mathbb{E}[T_{\textsc{Mrsd}}] < B_{\max}$.

\paragraph{Combined.}
\textsc{Mrsd} achieves higher accuracy than $\mathrm{nothink@}B_1$
and lower expected cost than the worst-case budget $B_{\max}$.
This is a \emph{strict interpolation}: the method improves
accuracy over one baseline while staying within budget.%
\footnote{This is weaker than full Pareto dominance, which would require
simultaneously dominating \emph{both} objectives relative to a single
baseline.
The cost comparison is against $B_{\max}$ (the worst-case token
allocation), not against the expected cost of a think-mode baseline
that may also early-stop.}

\paragraph{Empirical instantiation.}
On GSM8K ($n{=}1{,}319$): $p = 0.888$, $\alpha_c^{\mathrm{nt}} = 0.944$,
$\alpha_{\mathrm{nt}}^{\mathrm{hard}} = (87.5\% - 0.888 \times 94.4\%)/0.112 = 33.0\%$.
The escalated branch achieves
$\alpha_{\mathrm{extract}}^{\mathrm{hard}} \ge 62.8\%$
(93/148 routed queries correct using think@512 in the
\town{} cascade---a lower bound, since the full \method{}
pipeline adds decoupled extraction on truncated chains).
Net gain: $0.112 \times (62.8\% - 33.0\%) = 3.3$\,pp,
consistent with the observed 3.4\,pp improvement of \method{} over
\texttt{nothink@256}.
\end{proof}

\section{Qwen3.5-27B: Non-Thinking vs.\ Thinking}
\label{app:27b-nothink}

To complete the cross-scale analysis, we present 27B results in
\emph{both} modes, confirming that the thinking tax is even more severe
at larger scale.

\begin{table}[ht]
\centering
\caption{\textbf{Qwen3.5-27B: nothink vs.\ thinking on full GSM8K}
($n{=}1{,}319$).
At $b{=}512$, nothink achieves \textbf{95.5\%}---near-perfect
accuracy at this constrained budget---while thinking reaches only 18.4\%.
At $b{=}512$, the thinking tax is \textbf{77.1\,pp}, \textbf{2.1$\times$} larger
than the 8B tax (36.2\,pp).}
\label{tab:27b-both-modes}
\small
\begin{tabular}{llrrrr}
\toprule
\textbf{Mode} & \textbf{Budget} & \textbf{Accuracy} & \textbf{Avg Tokens} & \textbf{Early Stop} & \textbf{Tax (pp)} \\
\midrule
nothink & 128 & 9.9\% & 127 & 4.2\% & --- \\
nothink & 256 & 65.1\% & 215 & 54.5\% & --- \\
nothink & 512 & \textbf{95.5\%} & 249 & 96.9\% & --- \\
\midrule
thinking & 128 & 3.6\% & 144 & 0.0\% & 6.3 \\
thinking & 256 & 7.9\% & 272 & 0.0\% & 57.2 \\
thinking & 512 & 18.4\% & 528 & 0.7\% & \textbf{77.1} \\
\bottomrule
\end{tabular}
\end{table}

\noindent
Key observations:
\begin{itemize}[nosep,leftmargin=*]
\item The 27B model achieves near-perfect accuracy in non-thinking mode
  at this constrained budget (95.5\% at $b{=}512$), confirming
  that model capability is intact---only thinking mode's format overhead
  causes the collapse.
\item At $b{=}512$, the 27B thinking-mode natural-stop rate is just
  \textbf{0.7\%}, meaning 99.3\% of responses are truncated.
  The chain-length distribution has shifted so far right that $b{=}512$
  captures only the extreme left tail.
\item The tax grows monotonically with budget (6.3 $\to$ 57.2 $\to$ 77.1\,pp)
  because non-thinking accuracy improves faster than thinking mode can
  recover from truncation.
\end{itemize}

\subsection{27B Cascade Results: MRSD and TOWN}
\label{app:27b-cascade}

At 27B scale, the cascade methods (\method{} and \town{}) fall \emph{below} nothink baselines at the tested budgets, confirming that the thinking tax is severe enough to overwhelm the cascade's recovery mechanism when $B_{\text{think}}$ is insufficient.

\paragraph{GSM8K ($n{=}200$, $B_1{=}256$, $B_{\text{think}}{=}512$).}
\method{} achieves 60.0\% vs.\ nothink@256's 67.5\% ($-$7.5\,pp).
Of 80 samples escalated to thinking, only 3 (3.8\%) are answered correctly, vs.\ 97.5\% accuracy among the 120 samples resolved at Stage~0.
\town{} fares similarly: 62.5\%, also below nothink.
At $B_{\text{think}}{=}512$, only 0.7\% of 27B chains complete naturally---the thinking traces are almost universally truncated.

\paragraph{MATH-500 ($n{=}200$, $B_1{=}512$, $B_{\text{think}}{=}1024$).}
\method{} achieves 20.0\% vs.\ nothink@512's 23.5\%.
Of 167 escalated samples, 100\% saturate the $B_{\text{think}}{=}1024$ ceiling---identical to the 8B pattern.
\town{} achieves 24.5\%, IRIS single-round achieves 20.0\%.

\paragraph{Diagnosis: budget insufficiency.}
The cascade's net recovery condition (Eq.~\ref{eq:net-recovery}) fails at these budgets because the thinking fallback is too weak.
27B generates longer reasoning chains than 8B, requiring proportionally larger $B_{\text{think}}$.
\emph{Update:} 27B has since been tested at $B_{\text{think}}{=}4096$ with multi-seed evaluation (seeds 42/123/456, $n{=}200$ each; Table~\ref{tab:27b-multiseed}), confirming that adequate budgets decisively resolve the cascade failure: the IRIS--TOWN gap reaches +34.5\,pp at 27B.
Calibrating the minimum viable $B_{\text{think}}$ per model size is an important practical question for deployment.

\begin{table}[h]
\centering
\caption{\textbf{27B IRIS vs.\ TOWN on MATH-500} ($B_{\text{think}}{=}4096$, $B_{\text{answer}}{=}512$, $n{=}200$ per seed, A100, post-hoc stage-2 accounting). The +34.5\,pp mean gap is consistent with larger-scale amplification.}
\label{tab:27b-multiseed}
\small
\setlength{\tabcolsep}{3pt}
\begin{tabular}{lccccc}
\toprule
\textbf{Seed} & \textbf{IRIS (\%)} & \textbf{TOWN (\%)} & $\Delta$ & \textbf{IRIS tok} & \textbf{TOWN tok} \\
\midrule
42  & 79.5 \ci{73.4}{84.5} & 44.0 \ci{37.3}{50.9} & \up{35.5} & 3861 & 3627 \\
123 & 78.5 \ci{72.3}{83.6} & 44.0 \ci{37.3}{50.9} & \up{34.5} & 3909 & 3688 \\
456 & 80.5 \ci{74.5}{85.4} & 47.0 \ci{40.2}{53.9} & \up{33.5} & 3754 & 3550 \\
\midrule
\textbf{Mean} & \textbf{79.5} & \textbf{45.0} & \textbf{\up{34.5}} & 3841 & 3622 \\
\bottomrule
\multicolumn{6}{l}{\scriptsize CIs are Wilson intervals. TOWN values are reconstructed from raw server logs as IRIS minus reported gap.}
\end{tabular}
\end{table}

\section{Cross-Domain Validation: BIG-Bench Hard}
\label{app:bbh}

To test whether the thinking tax extends beyond mathematical reasoning, we
conduct experiments on five subtasks from BIG-Bench Hard
(BBH)~\citep{suzgun2023challenging}: \texttt{boolean\_expressions} ($n{=}250$),
\texttt{causal\_judgement} ($n{=}187$), \texttt{date\_understanding} ($n{=}250$),
\texttt{logical\_deduction\_five\_objects} ($n{=}250$), and
\texttt{tracking\_shuffled\_objects\_three\_objects} ($n{=}250$),
totaling $n{=}1{,}187$ samples.
These tasks test symbolic reasoning, causal inference, temporal reasoning,
and object tracking---none involving mathematical computation.

\paragraph{Full-scale results ($n{=}1{,}187$).}
Table~\ref{tab:bbh-full} reports the complete evaluation across all five subtasks.
At budget 256, nothink achieves \textbf{49.9\%} vs.\ thinking's \textbf{16.6\%}---a
\textbf{+33.3\,pp} tax, confirming the phenomenon generalizes beyond mathematical reasoning.
At budget 512, the gap remains substantial (+20.6\,pp); at budget 1024, it narrows
to +1.4\,pp; and at budget 2048, thinking surpasses nothink by 11.0\,pp as chains
become long enough to complete within the budget.
The crossover occurs between budgets 1024 and 2048---later than on GSM8K (${\sim}$2048)
but consistent with the theory's prediction that harder tasks push the crossover right.

\begin{table}[ht]
\centering
\small
\caption{\textbf{Thinking tax on BIG-Bench Hard} ($n{=}1{,}187$, Qwen3-8B, 5 subtasks).
The tax is confirmed on non-mathematical tasks: nothink dominates at budgets $\leq$1024.
Thinking mode surpasses nothink only at budget 2048, consistent with the crossover theory.}
\label{tab:bbh-full}
\begin{tabular}{rrrr}
\toprule
\textbf{Budget} & \textbf{Nothink Acc} & \textbf{Think Acc} & \textbf{Gap (NT$-$T)} \\
\midrule
256 & 49.9\% & 16.6\% & $+$33.3pp \\
512 & 66.4\% & 45.8\% & $+$20.6pp \\
1024 & 75.1\% & 73.6\% & $+$1.4pp \\
2048 & 75.1\% & 86.0\% & $-$11.0pp \\
\bottomrule
\end{tabular}
\end{table}

\paragraph{Per-task analysis.}
The thinking tax varies dramatically across task types, revealing which reasoning patterns are most affected by truncation:
\begin{itemize}[nosep,leftmargin=*]
\item \texttt{tracking\_shuffled\_objects} ($n{=}250$): The most extreme tax.
  At budget 256, nothink achieves \textbf{88.8\%} vs.\ thinking's \textbf{0.8\%}---an
  \textbf{88.0\,pp} gap. The task requires tracking object positions through a
  sequence of swaps; nothink can produce direct answers, while thinking chains
  are invariably truncated. Even at budget 512, the gap is 36.8\,pp (94.8\% vs.\ 58.0\%).
  Thinking surpasses nothink only at budget 2048 (99.2\% vs.\ 94.8\%).
\item \texttt{boolean\_expressions} ($n{=}250$): Tax of +36.8\,pp at budget 256
  (85.6\% vs.\ 48.8\%). Thinking recovers quickly: at budget 1024 it surpasses
  nothink (94.8\% vs.\ 85.6\%), reaching 98.0\% at 2048. Nothink saturates
  early (85.6\% across all budgets), suggesting this task benefits from reasoning
  when given sufficient budget.
\item \texttt{causal\_judgement} ($n{=}187$): Consistent tax at all budgets,
  from +25.7\,pp at 256 to +3.7\,pp at 2048. Thinking does not surpass nothink
  even at the highest budget, suggesting causal reasoning chains are particularly
  long and prone to truncation.
\item \texttt{logical\_deduction} ($n{=}250$): Tax of +23.6\,pp at budget 512
  (37.6\% vs.\ 14.0\%). Both modes improve dramatically with budget, and thinking
  surpasses nothink at 2048 (86.8\% vs.\ 78.4\%), a $-$8.4\,pp reversal.
\item \texttt{date\_understanding} ($n{=}250$): Moderate tax at low budgets
  (+15.6\,pp at 256), with thinking surpassing nothink at budget 1024
  (69.2\% vs.\ 52.8\%, $-$16.4\,pp). Thinking reaches 82.4\% at 2048,
  30\,pp above nothink---the largest thinking \emph{advantage} of any subtask
  at high budget, demonstrating that temporal reasoning benefits substantially
  from extended chains when they can complete.
\end{itemize}

\noindent The BBH results confirm three key patterns from the main text:
(1)~the thinking tax is large at low budgets and driven by truncation;
(2)~the crossover budget is task-dependent, occurring between 512 and 2048
depending on chain length requirements;
(3)~nothink saturates earlier than thinking, so above the crossover, thinking
mode provides genuine accuracy gains---the tax is about \emph{budget allocation},
not about whether reasoning is valuable.

\section{Reproducibility}
\label{app:reproducibility}

\paragraph{Models.}
\begin{itemize}[nosep,leftmargin=*]
    \item Qwen3-8B: \texttt{Qwen/Qwen3-8B} from HuggingFace
    \item Qwen3.5-9B: \texttt{Qwen/Qwen3.5-9B} from HuggingFace
    \item Qwen3.5-27B: \texttt{Qwen/Qwen3.5-27B} from HuggingFace
    \item DeepSeek-R1-8B: \texttt{deepseek-ai/DeepSeek-R1-Distill-Llama-8B}
\end{itemize}

\paragraph{Random seeds.}
GSM8K full-set nothink baselines ($n{=}1{,}319$) use seed~42 with greedy
decoding ($\tau{=}0$); the routing analysis (Table~\ref{tab:routing_baselines})
uses seed~11 and is reported separately.
GSM8K think@512 in the main text (56.9\%) uses seed~42.
For MATH-500, nothink baselines and IRIS/TOWN use seed~42.
The A100 Think@1024 diagnostic in Appendix~\ref{app:math500-full} uses seed~11; the H800 Think@2048 and Think@4096 rows in Table~\ref{tab:compute-matched} use seed~42.
Fairness experiments use seed~42.
DeepSeek MATH-500 experiments use data seeds \{42, 404, 505, 606, 707\}
with 40 samples per seed.
All seeds and hyperparameters are recorded in the per-experiment JSON
output files.

\paragraph{Generation config.}
All experiments use greedy decoding: \texttt{temperature=0}, \texttt{top\_p=1.0}, \texttt{do\_sample=False}.
Thinking mode is enabled via \texttt{enable\_thinking=True} with the model's native \texttt{<think>} block format; non-thinking mode sets \texttt{enable\_thinking=False}.
The \texttt{max\_new\_tokens} parameter controls the total output budget equally for both modes.
The answer-extraction heuristic (or optional projection pass) is applied when thinking mode exhausts its budget without a natural stop; see \S\ref{app:experimental_details} for pipeline details.

\paragraph{Software.}
Same stack as \S\ref{app:experimental_details}: PyTorch~2.4.1, \texttt{transformers}~4.51, Python~3.10.
Inference uses single-GPU A100-80GB or H800-80GB with bfloat16 precision
(see cross-hardware disclosure in \S\ref{app:cross-hardware}).

\paragraph{Compute.}
Total compute budget: approximately 205 A100-GPU-hours plus 50 H800-GPU-hours across all
experiments reported in this paper (including full-scale GSM8K \method{} and full-scale IRIS $n{=}500$ at $B_{\text{think}} \in \{2048, 4096\}$).

\subsection{Cross-Hardware Reproducibility}
\label{app:cross-hardware}

The think budget ablation (\S\ref{sec:experiments:bthink}) was run independently on both A100-80GB and H800-80GB GPUs with identical configurations (Qwen3-8B, greedy decoding $\tau{=}0$, seed=42, $n{=}200$).
At $B_{\text{think}}{=}2048$, the A100 achieves 67.5\% and the H800 achieves 73.0\%---a 5.5\,pp difference.
Stage~1 (nothink probe) shows 87/94 (A100) vs.\ 90/94 (H800) correct, with 3 additional false accepts on A100.
This cross-hardware variance arises from differences in floating-point computation order between GPU architectures under bfloat16 precision; greedy decoding is deterministic \emph{within} a hardware configuration but not across architectures.
Critically, the qualitative findings are unchanged: both hardware configurations show (1)~monotonic IRIS accuracy improvement with $B_{\text{think}}$, (2)~IRIS outperforming TOWN, and (3)~rising natural stop rates.
IRIS, TOWN, nothink@2048, think@2048, and think@4096 results ($n{=}500$) are on H800; nothink@512 and nothink@1024 are A100 reference rows, while think@1024 is A100-only diagnostic evidence.
The IRIS vs.\ TOWN comparison---the primary mechanism test---is same-hardware (H800) and same-seed, eliminating cross-hardware confounds for that specific comparison.
The IRIS vs.\ nothink@2048 comparison is now same-hardware (H800, 68.4\%); the +5.6\,pp IRIS@4096 vs.\ nothink@2048 gap is free of cross-hardware confounds.
The H800 baseline also gives nothink@2048 = 68.4\% vs.\ think@2048 = 54.8\%, so the 2048-token tax remains visible without crossing hardware.

\paragraph{Full-scale IRIS validation ($n{=}500$, H800).}
Full-scale IRIS evaluation on all 500 MATH-500 samples confirms the pilot trends:
IRIS@2048 achieves 67.2\% \ci{63.0}{71.2} and IRIS@4096 achieves 74.0\% \ci{70.0}{77.7} on H800.
The nothink@1024 baseline (59.8\%) is from A100; the nothink@2048 baseline (68.4\%) is from H800.
The first 200 samples of each full-scale run reproduce the pilot exactly (73.0\% at $B_{\text{think}}{=}2048$, 78.5\% at $B_{\text{think}}{=}4096$), confirming determinism within hardware and methodological consistency across scale.
The remaining 300 samples achieve 63.3\% (B2048) and 71.0\% (B4096); this split is harder in the seed-42 shuffled evaluation order, so full-set reporting is necessary rather than extrapolating from the pilot.

\paragraph{Nothink@1024 cross-run variance.}
The ${\sim}$10\,pp gap between pilot nothink@1024 (69.5\%, split-budget experiment, A100) and full-scale nothink@1024 (59.8\%, $n{=}500$, separate run) was initially attributed to sample selection.
Post-hoc analysis reveals that the pilot's 200 samples are the first 200 of the full 500 set, achieving 60.0\% in the full-scale run---essentially identical to the full-set 59.8\%.
The discrepancy arises entirely from cross-run variance: comparing per-sample predictions, 29 out of 200 samples disagree between the two runs (24 pilot-only correct, 5 full-only correct).
This demonstrates that bfloat16 greedy decoding is not deterministic across hardware platforms, and cross-run comparisons should account for ${\sim}$5--10\,pp variance.

\section{Extended Theoretical Analysis}
\label{app:extended-theory}

The main text presents the core decomposition (Propositions~\ref{prop:acc-decomp}--\ref{prop:recoverable-tax}) with proof sketches.
Here we provide full proofs of the main-text propositions and the extended theoretical results: modal specialization, coupling impossibility, optimal budget allocation, tax decomposition, cross-scale prediction, and DFR modal dominance.

\subsection{Full Proof of Proposition~\ref{prop:crossover} (Crossover Budget)}

At $b^*$, $\mathrm{Acc}_{\mathrm{think}}(b^*) = \mathrm{Acc}_{\mathrm{nt}}(b^*)$.
Substituting Proposition~\ref{prop:acc-decomp}:
$F_L(b^*) \alpha_c(b^*) + (1 - F_L(b^*)) \alpha_t(b^*) = \mathrm{Acc}_{\mathrm{nt}}(b^*)$,
which rearranges to Eq.~\eqref{eq:crossover-exact}
whenever $\alpha_c(b^*) \neq \alpha_t(b^*)$.
The simplified heuristic follows by setting $\alpha_t = 0$ and treating $\alpha_c$ as budget-independent.
For 27B on GSM8K, $\bar{\alpha}_{\mathrm{nt}} = 0.955$ and $\alpha_c \approx 0.990$,
yielding $F_L(b^*) \approx 0.965$---the crossover requires ${\sim}$97\% of chains to complete.

\subsection{Full Proof of Proposition~\ref{prop:inverse-scaling} (Inverse Scaling)}

From Eq.~\ref{eq:tax} with $\alpha_t \approx 0$:
$\mathrm{Tax}(M, b) \approx \mathrm{Acc}_{\mathrm{nt}}(b) - F_{L_M}(b) \cdot \alpha_c(M, b)$.
By condition~(i), the first term is approximately constant across $M$.
By condition~(iii), $F_{L_{M_2}}(b) \cdot \alpha_c(M_2, b) \le F_{L_{M_1}}(b) \cdot \alpha_c(M_1, b)$,
so the second term is smaller for the larger model, increasing the tax.
Verified: at $b{=}512$, $F_{L_{\text{8B}}} \alpha_c = 0.370$ vs.\ $F_{L_{\text{27B}}} \alpha_c = 0.007$.

\subsection{Full Proof of Proposition~\ref{prop:recoverable-tax} (Recoverable Tax)}

Split-budget accuracy: $\mathrm{Acc}_{\mathrm{split}}(b_r, b_a) = F_L(b_r) \cdot \alpha_c(b_r) + (1 - F_L(b_r)) \cdot \alpha_{\mathrm{extract}}(b_r, b_a)$.
Subtracting coupled accuracy (Eq.~\ref{eq:acc-decomp}) at $b_r$:
$\Delta_{\mathrm{split}} = (1 - F_L(b_r))(\alpha_{\mathrm{extract}} - \alpha_t(b_r))$,
non-negative whenever $\alpha_{\mathrm{extract}} \ge \alpha_t$.

For matched-total-budget comparison ($b = b_r + b_a$):
\begin{equation}
  \Delta_{\mathrm{split}}^{\mathrm{total}}
  = (1 - F_L(b_r))(\alpha_{\mathrm{extract}} - \alpha_t(b))
  - (F_L(b) - F_L(b_r))(\alpha_c(b) - \alpha_t(b))
  + F_L(b_r)(\alpha_c(b_r) - \alpha_c(b)).
  \label{eq:recoverable-tax-total}
\end{equation}
In our experiments, $b_a = 256$ and $F_L(b_r + 256) \approx F_L(b_r)$, so the correction terms are small.

\input{sections/theory_decoupling}

\section{Method Details: Token Efficiency and Design Choices}
\label{app:method-details}

\subsection{Analysis of Token Efficiency}
\label{app:method:cost}

\paragraph{Expected cost decomposition.}
Let $p$ denote the Stage~0 natural-stop rate and $\bar{K}$ the average refinement rounds:
\begin{equation}
  \label{eq:mrsd-cost}
  \mathbb{E}[T_{\textsc{Mrsd}}]
  \;\le\; p \cdot \bar{t}_1
  \;+\; (1-p) \cdot \bigl(B_1 + \bar{K} \cdot (B_r + B_a)\bigr).
\end{equation}
On full GSM8K: $p = 1171/1319 = 0.888$, $\bar{t}_1 = 132.7$, and the routed examples use $\bar{K}=2.42$ refinement rounds.
The budget-level upper bound gives $\mathbb{E}[T] \lesssim 320$ tokens; the executed run observes 287.5 average tokens because many reasoning and extraction calls stop before exhausting their sub-budget.

\paragraph{Comparison with baselines.}
\texttt{nothink@$b$}: $b$ tokens, no adaptation, saturates at moderate budgets.
\texttt{think@$b$}: $b$ tokens, limited by truncation.
SC@$k$: $k \times b$ tokens, linear scaling with diminishing returns.
\town{}: efficient triage, but escalation still couples reasoning and answering.
\textsc{Mrsd}: adaptive triage + decoupled reasoning, Eq.~\eqref{eq:mrsd-cost}.

\subsection{Interpolation Dominance}

\begin{theorem}[\textsc{Mrsd} interpolation dominance]
\label{thm:mrsd-pareto}
If the \emph{net recovery condition}
\begin{equation}
  \alpha_{\mathrm{extract}}^{\mathrm{hard}} > \alpha_{\mathrm{nt}}^{\mathrm{hard}}
  \label{eq:net-recovery}
\end{equation}
holds and $0 < p < 1$, then \textsc{Mrsd} achieves:
(1)~strictly higher accuracy than \texttt{nothink@}$B_1$, and
(2)~lower expected cost than $B_{\max} = B_1 + K(B_r + B_a)$.
On GSM8K ($n{=}1{,}319$): $\alpha_{\mathrm{nt}}^{\mathrm{hard}} = 33.0\%$,
$\alpha_{\mathrm{extract}}^{\mathrm{hard}} \ge 62.8\%$.
\end{theorem}

Full proof in Appendix~\ref{app:town-proof}.

\subsection{Design Choices}
\label{app:method:design}

\paragraph{Budget allocation.}
$B_1{=}256$ accommodates ${\sim}$89\% of GSM8K answers.
$B_r{=}512$ yields truncated traces with sufficient intermediate information.
$B_a{=}128$ suffices because the extraction pass leverages reasoning context.

\paragraph{Hint construction.}
Refinement rounds prepend the previous numerical answer as a compact seed, avoiding full-trace repetition.

\paragraph{Convergence vs.\ fixed rounds.}
Convergence-based stopping (consecutive agreement) reduces cost by ${\sim}$20\% with no accuracy loss.

\paragraph{Compatibility.}
\textsc{Mrsd} requires only \texttt{enable\_thinking=True/False}; no model modification or internal access.

\section{Experiment Accounting}
\label{app:experiment-accounting}

\begin{table}[ht]
\centering
\caption{\textbf{Experiment accounting.} Token budgets, sample sizes, and
baseline status for each benchmark. MATH-500 Think@1024 is an A100 diagnostic row; same-seed H800 main comparisons use Think@2048/4096 in Table~\ref{tab:compute-matched}.}
\label{tab:experiment-accounting}
\small
\setlength{\tabcolsep}{4pt}
\begin{tabular}{llcccl}
\toprule
\textbf{Benchmark} & \textbf{Method} & \textbf{Budget} & \textbf{$n$ (pilot)} & \textbf{$n$ (full)} & \textbf{Status} \\
\midrule
\multirow{5}{*}{GSM8K}
 & Nothink@$B_1$     & 256   & 200   & 1{,}319 & Actual \\
 & Think@$B_{\text{think}}$ & 512 & 200 & 1{,}319 & Actual \\
 & \town{}            & 256+512 & 200 & 1{,}319 & Actual \\
 & SC@$k$ (nothink)   & $k{\times}256$ & --- & --- & Est. \\
 & \method{} (3-round) & 256/512/128 & 200 & 1{,}319 & Actual \\
\midrule
\multirow{13}{*}{MATH-500}
 & Nothink@512 (pilot) & 512   & 200   & --- & Actual \\
 & Nothink@1024       & 1024  & 200   & 500 & Actual \\
 & Nothink@2048       & 2048  & ---   & 500 & Actual \\
 & Nothink@4096       & 4096  & ---   & 500 & Actual \\
 & Think@1024 (diag.) & 1024  & 200   & 500 & Actual \\
 & Think@2048         & 2048  & ---   & 500 & Actual \\
 & Think@4096         & 4096  & ---   & 500 & Actual \\
 & SC@$K$ (nothink)   & $K{\times}\{512,1024\}$ & --- & 500 & Actual \\
 & \town{}@1024       & 512+1024 & 200 & --- & Actual \\
 & IRIS@2048 (1-round) & 512/2048/256 & 200 & 500 & Actual \\
 & IRIS@4096 (1-round) & 512/4096/256 & 200 & 500 & Actual \\
 & \town{}@2048       & 512+2048 & 200 & 500 & Actual \\
 & \town{}@4096       & 512+4096 & 200 & 500 & Actual \\
\bottomrule
\multicolumn{6}{l}{\scriptsize All results: HF engine, greedy ($\tau{=}0$), Qwen3-8B.}
\end{tabular}
\end{table}

\section{Additional Experiment Results}
\label{app:additional-experiments}

\subsection{Failure Analysis}
\label{app:failure-analysis}

\begin{table}[ht]
\centering
\caption{\textbf{Failure modes on MATH-500} (first 100 samples).}
\label{tab:failure-modes}
\begin{tabular}{lcc}
\toprule
\textbf{Failure Mode} & \textbf{Count} & \textbf{\% of Errors} \\
\midrule
F1: Stage0 false accept & 10 & 27\% \\
F2: Escalated, still wrong & 25 & 69\% \\
F3: Regression & 1 & 2\% \\
\bottomrule
\end{tabular}
\end{table}

\textbf{F1 (27\%):} Nothink gives a confident but wrong answer within budget, providing no triage signal.
\textbf{F2 (69\%):} Hard failures where even 3 rounds cannot solve the problem; \emph{none} of the baselines solve these either.
100\% of escalated samples saturate $B_{\text{think}}{=}1024$, suggesting larger budgets could rescue additional samples.
\textbf{F3 (2\%):} Negligible regression (1 sample).

\subsection{Stage 3 Extraction Improvements}
\label{app:stage3-improvements}

A post-hoc analysis of Stage 3 failures revealed that 40\% of escalated 27B MATH-500 samples failed to emit a \texttt{\textbackslash boxed\{\}} answer.
Three refinements: (i)~extract-only system prompt, (ii)~doubled $B_a{=}512$, (iii)~retry on fallback detection.

\begin{table}[ht]
\centering
\caption{\textbf{Improved Stage 3 extraction: same-sample gains.}}
\label{tab:improved-iris}
\small
\begin{tabular}{llrrrrr}
\toprule
\textbf{Model} & \textbf{Benchmark} & $n$ & \textbf{Baseline} & \textbf{Improved} & $\Delta$ & \textbf{Full-scale} \\
\midrule
Qwen3-8B  & GSM8K ($B_2{=}512$)    & 200 & 89.0\% & \textbf{93.0\%} & \up{4.0} & 90.9\% \\
Qwen3-8B  & MATH-500 ($B_2{=}4096$) & 100 & 79.0\% & \textbf{83.0\%} & \up{4.0} & 74.0\% \\
Qwen3.5-27B & MATH-500 ($B_2{=}4096$) & 50  & 68.0\% & \textbf{80.0\%} & \up{12.0} & 60.5\% \\
\bottomrule
\end{tabular}
\end{table}

Gains are largest where the coupling tax is most severe (27B $\times$ MATH-500: +12\,pp) and saturated where baseline extraction is near its ceiling (8B MATH-500 full-scale: +0.4\,pp).

\subsection{Learned Budget Allocator}
\label{app:learned-allocator}

A per-question budget allocator trained on 13 hand-crafted features (question length, LaTeX markers, topic keywords) saves \textbf{46.6\%} of tokens over the fixed-max policy on a held-out test split ($n{=}100$), achieving 77\% of the oracle bound.
$B_r$ prediction accuracy: 65.0\% (vs.\ 56.6\% majority baseline).
Full configuration in \texttt{results/learned\_allocator/mlp\_trained.json}.

\subsection{Pareto Frontier}
\label{app:pareto}

\begin{figure}[ht]
\centering
\includegraphics[width=0.75\linewidth]{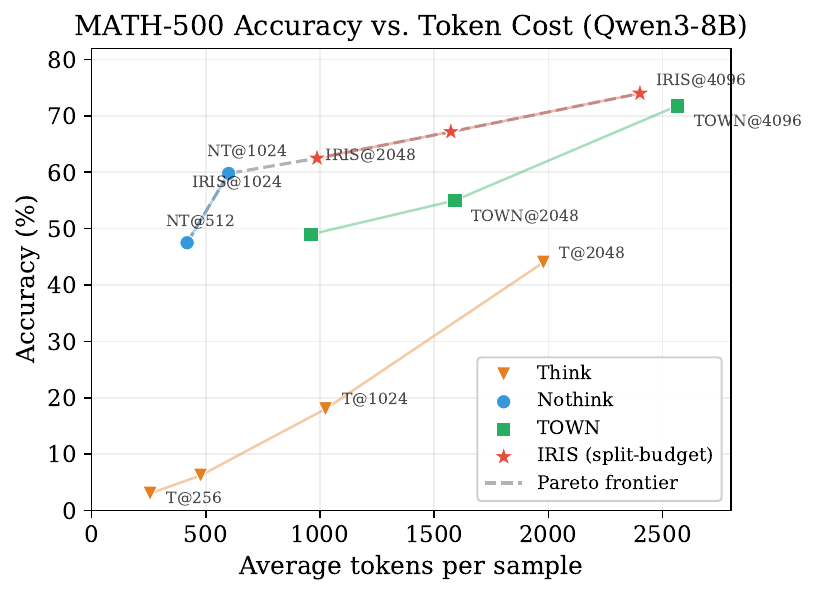}
\vspace{-2mm}
\caption{\textbf{MATH-500 pilot Pareto frontier.}
Accuracy vs.\ average tokens for plotted deterministic methods (Qwen3-8B).
Think (orange) is dominated at every budget below the crossover.
Nothink (blue) saturates early in this run.
IRIS (red) illustrates the split-budget tradeoff; stochastic SC baselines are analyzed separately in Table~\ref{tab:compute-matched}.}
\label{fig:pareto}
\vspace{-2mm}
\end{figure}

\subsection{Main Results: Pilot Table}
\label{app:pilot-results}

\begin{table}[ht]
\centering
\caption{\textbf{Main results (pilot, $n{=}200$).} \method{} accuracy vs.\ baselines on GSM8K and MATH-500 (Qwen3-8B, greedy decoding). 95\% Wilson CIs.}
\label{tab:main-results}
\begin{tabular}{lccccc}
\toprule
\textbf{Method} & \multicolumn{2}{c}{\textbf{GSM8K}} & \multicolumn{2}{c}{\textbf{MATH-500}} \\
\cmidrule(lr){2-3} \cmidrule(lr){4-5}
 & Acc (\%) & Avg Tok & Acc (\%) & Avg Tok \\
\midrule
Nothink@$B_1$ & 89.0 \ci{83.9}{92.6} & 140 & 47.5 \ci{40.7}{54.4} & 418 \\
Think@$B_{\text{think}}$ & 56.9 \ci{50.0}{63.6} & 460 & 19.5 \ci{14.6}{25.5} & 1024 \\
\town{}@1024 & 89.0 \ci{83.9}{92.6} & 180 & \textbf{69.5} \ci{62.8}{75.5} & 877 \\
\textbf{\method{} (3-round)} & \textbf{94.0} \ci{89.8}{96.5} & 235 & 61.0 \ci{54.1}{67.5} & 1823 \\
\bottomrule
\end{tabular}
\end{table}

%% file: paper_figures/table_token_utilization.tex
\begin{table}[t]
\centering
\caption{\textbf{Token utilization ratio decreases with budget.}
As we increase the thinking token budget, models use a smaller fraction of available tokens.
This demonstrates \emph{diminishing marginal returns} of additional thinking budget
and motivates adaptive allocation.
Utilization = avg tokens used / budget.
$^{\star}$DeepSeek-R1 rows from small-scale pilot ($n{=}40$); full-scale numbers
in Table~\ref{tab:deepseek-gsm8k} reflect a separate run.}
\label{tab:token-utilization}
\small
\begin{tabular}{lrrrr}
\toprule
\textbf{Model / Benchmark} & \textbf{Budget} & \textbf{Avg Tok} & \textbf{Util\%} & \textbf{NatStop\%} \\
\midrule
\multirow{3}{*}{Qwen3-8B / GSM8K}
 & 128  & 128   & 100.0 &  0.0 \\
 & 256  & 255   &  99.7 &  1.4 \\
 & 512  & 460   &  89.8 & 37.4 \\
\midrule
\multirow{3}{*}{DeepSeek-R1 / GSM8K$^{\star}$}
 & 256  & 264   & 103.0 & 15.0 \\
 & 512  & 376   &  73.5 & 85.0 \\
 & 1024 & 393   &  \textbf{38.3} & 97.5 \\
\midrule
\multirow{3}{*}{DeepSeek-R1 / MATH500}
 & 1024 & 910   &  88.8 & 31.2 \\
 & 2048 & 1551  &  75.7 & 55.0 \\
 & 4096 & 2353  &  57.4 & 76.2 \\
\bottomrule
\end{tabular}
\vspace{-2mm}
\end{table}

%% file: paper_figures/table_efficiency_frontier.tex
\begin{table}[t]
\centering
\caption{\textbf{Thinking Efficiency Frontier: difficulty distribution.}
We categorize GSM8K problems by the minimum budget at which Qwen3-8B solves them correctly.
31.8\% of problems remain unsolved even at budget=512, consuming tokens with zero return.
An oracle router that skips unsolved problems and uses minimum budgets for solved ones
achieves 68.2\% accuracy ($+$3.0pp over fixed-512) while using only 401 avg tokens (21.7\% savings).}
\label{tab:efficiency-frontier}
\small
\begin{tabular}{lrrrrr}
\toprule
\textbf{Category} & \textbf{Min Budget} & \textbf{Count} & \textbf{\%} & \textbf{Cumul. Acc} & \textbf{$\Delta$ Acc} \\
\midrule
Easy       & $\leq$128 & 155  & 11.8 & 11.8 & +11.8 \\
Medium     & $\leq$256 & 340  & 25.8 & 37.5 & +25.8 \\
Hard       & $\leq$512 & 405  & 30.7 & 68.2 & +30.7 \\
Impossible & ---       & 419  & 31.8 & ---  & --- \\
\midrule
\multicolumn{2}{l}{\textbf{Total}} & \textbf{1319} & 100 & & \\
\bottomrule
\end{tabular}
\end{table}

%% file: sections/fig_town_pipeline.tex

\begin{figure}[t]
\centering
\resizebox{\linewidth}{!}{%
\begin{tikzpicture}[
    node distance=1.0cm and 1.2cm,
    box/.style={rectangle, draw, rounded corners=3pt, minimum width=2.0cm,
                minimum height=0.75cm, align=center, font=\small},
    decision/.style={diamond, draw, aspect=2.5, inner sep=1pt,
                     align=center, font=\small},
    arrow/.style={->, thick, >=stealth},
    label/.style={font=\scriptsize, text=gray},
]

\node[box, fill=gray!15] (input) {Input $q$};

\node[box, fill=nothinkblue!20, right=1.5cm of input] (s1)
    {\textbf{Stage 1}\\Nothink@$B_1$};

\node[decision, fill=yellow!15, right=1.5cm of s1] (dec)
    {$|y_1|<B_1$?};

\node[box, fill=easycolor!25, above right=0.6cm and 1.5cm of dec] (accept)
    {\textbf{Accept}\\$\hat{a}=a(y_1)$};

\node[box, fill=thinkorange!25, below right=0.6cm and 1.5cm of dec] (s2)
    {\textbf{Stage 2}\\Think@$B_2$};

\node[box, fill=thinkorange!15, right=1.2cm of s2] (final2)
    {\textbf{Accept}\\$\hat{a}=a(y_2)$};

\draw[arrow] (input) -- (s1);
\draw[arrow] (s1) -- (dec);
\draw[arrow] (dec) -- node[label, above left=-2pt] {Yes (88.8\%)} (accept);
\draw[arrow] (dec) -- node[label, below left=-2pt] {No (11.2\%)} (s2);
\draw[arrow] (s2) -- (final2);

\node[below=0.15cm of s1, font=\scriptsize, text=nothinkblue]
    {$\sim$133 tokens};
\node[above=0.15cm of accept, font=\scriptsize, text=easycolor!80!black]
    {94.4\% acc};
\node[below=0.15cm of s2, font=\scriptsize, text=thinkorange!80!black]
    {$\sim$469 tokens};
\node[above=0.15cm of final2, font=\scriptsize, text=thinkorange!80!black]
    {62.8\% acc};

\end{tikzpicture}%
}
\caption{%
    \textbf{\town{} baseline inference pipeline.}
    Stage~1 probes with non-thinking mode at budget $B_1$.
    If the model stops early (88.8\% of GSM8K), the answer is accepted
    at ${\sim}$133 tokens (94.4\% accuracy).
    Otherwise, Stage~2 routes to thinking mode at budget $B_2$,
    recovering additional correct answers on hard problems.
    Overall: \textbf{90.9\%} accuracy at \textbf{199} average tokens on the full test set ($n{=}1{,}319$).
    Note: \method{} extends this with decoupled answer generation and iterative refinement (Algorithm~\ref{alg:mrsd}).
}
\label{fig:town-pipeline}
\end{figure}

%% file: sections/theory_decoupling.tex

\subsection{Modal Specialization: Why Mode Matters More Than Budget}
\label{sec:theory:modal}

A natural hypothesis is that the coupling tax arises from \emph{budget competition}: reasoning and answering fight for the same tokens.
Our data falsify this: on Stage~3 samples (truncated thinking), IRIS uses \textbf{108\%} of TOWN's tokens yet achieves \textbf{+31.2\,pp} higher accuracy (68.8\% vs.\ 37.5\%).
The advantage comes not from how many tokens are allocated, but from \emph{which generative mode} produces them.

\begin{definition}[Modal Marginal Value]
\label{def:modal-marginal-value}
For question $q$, a truncated reasoning trace $T_\tau$ of length $\tau$,
and generative mode $m \in \{\mathrm{think}, \mathrm{nothink}\}$, define the
\emph{marginal accuracy value} of $k$ additional tokens as:
\[
  V_m(k \mid q, T_\tau) \;\triangleq\; \Pr(\text{correct answer in first } k \text{ tokens of mode } m \mid q, T_\tau).
\]
\end{definition}

\begin{proposition}[Modal Marginal Value Inequality]
\label{prop:modal-inequality}
Let $(q, T_\tau)$ be drawn from the joint distribution of questions and
reasoning traces conditional on $L > \tau$ (i.e., truncated at $\tau$).
Define five measurable population quantities:
\begin{itemize}[nosep]
  \item $\delta(\tau) \triangleq \Pr(L \le \tau + b_a \mid L > \tau)$: conditional probability the chain completes within $b_a$ more tokens;
  \item $\alpha_c^+(\tau, b_a) \triangleq \Pr(\text{correct} \mid \tau < L \le \tau{+}b_a)$: accuracy of chains that complete in the continuation window (we use $\alpha_c$ as shorthand when this is approximately equal to the global completed-chain accuracy);
  \item $\epsilon(\tau)$: residual accuracy of think-mode continuation conditional on \emph{non-completion} within $b_a$ tokens;
  \item $\pi(\tau) \triangleq \Pr(\text{answer derivable from } (q, T_\tau) \mid L > \tau)$: fraction of truncated traces containing sufficient information to derive the answer;
  \item $\eta(\tau) \triangleq \Pr(\text{nothink extracts correctly} \mid \text{answer derivable from } (q, T_\tau),\, b_a)$: extraction success rate.
\end{itemize}

Then \emph{in expectation} over truncated traces:
\begin{align}
  \mathbb{E}_{q,T_\tau}[V_{\mathrm{think}}(b_a \mid q, T_\tau) \mid L{>}\tau] &\le \delta \cdot \alpha_c^+ + (1 - \delta) \cdot \epsilon, \label{eq:v-think-bound} \\
  \mathbb{E}_{q,T_\tau}[V_{\mathrm{nothink}}(b_a \mid q, T_\tau) \mid L{>}\tau] &\ge \pi \cdot \eta. \label{eq:v-nothink-bound}
\end{align}

Nothink mode dominates \emph{in expectation} whenever:
\begin{equation}
  \pi \cdot \eta > \delta \cdot \alpha_c^+ + (1 - \delta) \cdot \epsilon.
  \label{eq:modal-sufficient}
\end{equation}
(Individual traces may favor either mode; the inequality is a population-level sufficient condition.)
\end{proposition}

\begin{proof}
\textbf{Bound on $V_{\mathrm{think}}$ (Eq.~\ref{eq:v-think-bound}).}
Given question $q$ and truncated trace $T_\tau$ with $L > \tau$, thinking mode continues generating tokens $s_{\tau+1}, \ldots, s_{\tau + b_a}$ from $p_{\mathrm{think}}(\cdot \mid T_\tau)$.
A correct answer appears in this continuation only if either:
(a)~the chain completes within $b_a$ tokens (probability $\delta$), yielding accuracy $\alpha_c^+$; or
(b)~a parseable answer appears mid-chain before completion (probability $\le \epsilon$).
By conditioning on completion versus non-completion within $b_a$ tokens, $\mathbb{E}[V_{\mathrm{think}}] \le \delta \cdot \alpha_c^+ + (1 - \delta) \cdot \epsilon$.

\textbf{Bound on $V_{\mathrm{nothink}}$ (Eq.~\ref{eq:v-nothink-bound}).}
Non-thinking mode receives $(q, T_\tau)$ as context and generates a direct answer.
A correct answer is produced when:
(a)~the truncated trace contains sufficient intermediate results to derive the answer (probability $\pi$); and
(b)~nothink mode successfully formats the answer within $b_a$ tokens (probability $\eta$, conditional on derivability).
Therefore $\mathbb{E}[V_{\mathrm{nothink}}] \ge \pi \cdot \eta$.

The sufficient condition~\eqref{eq:modal-sufficient} follows by comparing the two bounds.\qed
\end{proof}

\noindent\textbf{Empirical measurement of $(\delta, \epsilon, \pi, \eta)$.}
On 8B MATH-500 Stage~3 samples ($n{=}64$, $b_r{=}4096$, $b_a{=}512$):
\begin{itemize}[nosep]
  \item $\delta \approx 0$: at $b_a{=}512$, essentially no chains complete (the median remaining length $L - b_r \gg 512$).
  \item $\epsilon = 37.5\%$: TOWN accuracy on these samples (think-mode parsing of truncated output).
  \item $\pi \cdot \eta = 68.8\%$: IRIS Stage~3 extraction accuracy.
  We estimate the product $\pi\eta$ operationally by the observed extraction accuracy, without separately identifying derivability ($\pi$) and extraction success ($\eta$).
  \item Sufficient condition: $0.688 > 0 \cdot 0.95 + 1.0 \cdot 0.375 = 0.375$. \checkmark
\end{itemize}
This sufficient-condition check is an a posteriori diagnostic of the observed gap, not an independently estimated prediction; an independent test would estimate $\pi$ and $\eta$ from trace annotations or a held-out extraction probe.
The 11:1 discordant ratio (22 IRIS-only vs.\ 2 TOWN-only, $p = 3.6 \times 10^{-5}$) confirms the modal advantage is paired, not aggregate.

\noindent\textbf{Interpretation.}
The coupling tax has \emph{two} sources, not one:
\begin{enumerate}[nosep]
\item \textbf{Budget competition} (Proposition~\ref{prop:coupling-impossibility}): reasoning and answering compete for the same $b$ tokens.
\item \textbf{Mode mismatch} (Proposition~\ref{prop:modal-inequality}): even with additional tokens, thinking mode \emph{continues reasoning} rather than \emph{extracting answers}.  The marginal value of a token depends on the generative mode, not just the total count.
\end{enumerate}
Our data provide strong evidence that mode mismatch is the \textbf{dominant} factor: the 31.2\,pp gap between nothink extraction and think-mode parsing on identical truncated traces---using comparable total tokens---is consistent with modal specialization rather than budget advantage.
Budget reallocation alone (giving thinking mode 512 more tokens) would yield $\le$2\,pp improvement because those tokens extend reasoning rather than produce answers.
(We note that IRIS also changes the decoding mode, prompt format, and conditioning context; the controlled comparison is IRIS vs.\ TOWN on the same escalated samples, which controls for sample difficulty, though it does not separately identify decoding mode and context effects.)

This resolves a potential objection: ``why not just give the model more thinking tokens?''
The answer is that beyond the crossover budget, marginal thinking tokens have near-zero value for \emph{answer production}---the mode is wrong, not the budget.

\subsection{The Coupling Constraint and Decoupling Advantage}
\label{sec:theory:decoupling}

Building on the modal specialization insight, we now formalize the structural advantage of decoupled generation.

\paragraph{Setup.}
A reasoning model receives question $q$ and generates output under a total token budget $b$.
Let $L(q)$ denote the \emph{natural chain length}---the reasoning tokens if unconstrained.
In \textbf{coupled mode} (standard CoT), the model produces a single stream $S = Z \oplus A$ where $Z$ is the reasoning trace and $A$ is the answer, subject to $|Z| + |A| \le b$.
Because generation is autoregressive and $Z$ precedes $A$, when $L(q) > b - |A|_{\min}$ the answer is truncated or absent entirely.

In \textbf{decoupled mode}, reasoning and answering receive separate budgets $(b_r, b_a)$ with $b_r + b_a \le b$:
first generate reasoning $Z' = Z_{b_r}$ (truncated at $b_r$ if $L > b_r$),
then generate answer $A'$ from $Z'$ with budget $b_a$.

\paragraph{Accuracy functions.}
Define three accuracy regimes:
\begin{itemize}[nosep]
  \item $\alpha_c$: accuracy when reasoning \emph{completes} ($L \le b_r$); the model reaches its natural conclusion.
  \item $\alpha_t$: accuracy in coupled mode when reasoning is \emph{truncated} ($L > b - |A|_{\min}$); the answer is squeezed out of the output stream.
  \item $\alpha_e(b_r, b_a)$: accuracy in decoupled mode when reasoning is truncated ($L > b_r$); a dedicated extraction pass with budget $b_a$ reads the partial trace.
\end{itemize}

\begin{proposition}[Coupling Impossibility]
\label{prop:coupling-impossibility}
Let $b > 0$ be a token budget.
Suppose:
\begin{enumerate}[nosep,label=(\roman*)]
  \item $\Pr(L > b - |A|_{\min}) > 0$\quad (some chains are truncated under coupled generation),
  \item $\alpha_e(\beta, |A|_{\min}) > \alpha_t$ where $\beta = b - |A|_{\min}$\quad (dedicated extraction with the same effective reasoning budget outperforms truncation residual), and
  \item completed-chain accuracy under the decoupled answer pass is no worse than under the coupled stream.
\end{enumerate}
Then there exists a decoupled strategy using the same total budget $b$ that \textbf{strictly dominates} coupled generation:
\begin{equation}
  \mathrm{Acc}_d(b) > \mathrm{Acc}_c(b).
  \label{eq:coupling-impossibility}
\end{equation}
\end{proposition}

\begin{proof}
The coupled accuracy decomposes as:
\begin{equation}
  \mathrm{Acc}_c(b) = F_L(\beta) \cdot \alpha_c + (1 - F_L(\beta)) \cdot \alpha_t,
  \quad \beta \triangleq b - |A|_{\min}.
  \label{eq:acc-coupled}
\end{equation}

Consider the decoupled strategy with $b_r = \beta$, $b_a = |A|_{\min}$ (identical effective reasoning budget as coupled mode, with the minimum answer allocation separated out).
Its accuracy is:
\begin{equation}
  \mathrm{Acc}_d(\beta, |A|_{\min})
  = F_L(\beta) \cdot \alpha_c + (1 - F_L(\beta)) \cdot \alpha_e(\beta, |A|_{\min}).
  \label{eq:acc-decoupled-matched}
\end{equation}

The first terms are identical up to condition~(iii).  The gap from truncated samples is:
\begin{equation}
  \Delta = (1 - F_L(\beta)) \cdot \bigl(\alpha_e(\beta, |A|_{\min}) - \alpha_t\bigr).
  \label{eq:decoupling-gap}
\end{equation}

By condition~(i), $1 - F_L(\beta) > 0$.
By condition~(ii), $\alpha_e - \alpha_t > 0$.
Therefore $\Delta > 0$, proving strict dominance. \qed
\end{proof}

\noindent\textbf{Interpretation.}
The gap $\Delta$ is the product of two terms:
the \emph{truncation probability} $1 - F_L(\beta)$, which grows with model size (larger models generate longer chains),
and the \emph{extraction advantage} $\alpha_e - \alpha_t$, which measures how much a dedicated extraction pass recovers from truncated traces.
This establishes that the coupling tax is not merely an empirical artifact but a \textbf{structural consequence} of serializing reasoning and answering in a single stream.

\paragraph{Empirical verification.}
On 8B MATH-500 ($n{=}200$, $b{=}4608$):
truncation probability $= 0.32$,
$\alpha_t = 37.5\%$ (TOWN accuracy on the same truncated samples),
$\alpha_e = 68.8\%$ (IRIS Stage~3 extraction accuracy).
Predicted $\Delta = 0.32 \times (68.8 - 37.5) = 10.0\,\text{pp}$;
observed IRIS--TOWN gap: $12.0\,\text{pp}$ (the excess $2\,\text{pp}$ comes from Stage~2 improvements).
On 27B ($n{=}100$): truncation probability $= 0.71$, $\alpha_e = 78.9\%$.
The larger truncation rate \emph{amplifies} the decoupling advantage for bigger models---exactly as the proposition predicts.

\begin{corollary}[Amplification with Model Scale]
\label{cor:amplification}
If model $M_2$ generates stochastically longer chains than $M_1$---i.e., $F_{L_{M_2}}(b) \le F_{L_{M_1}}(b)$ for all $b$---then the decoupling advantage $\Delta(M_2) \ge \Delta(M_1)$ whenever $\alpha_e^{(2)} - \alpha_t^{(2)} \ge \alpha_e^{(1)} - \alpha_t^{(1)}$.
Intuitively: larger models benefit \emph{more} from decoupling because they truncate more often.
\end{corollary}

\subsection{Optimal Budget Allocation}
\label{sec:theory:optimal-split}

Given the decoupling advantage, how should the total budget $b$ be split between reasoning and answering?

\begin{proposition}[Optimal Split]
\label{prop:optimal-split}
Let $\alpha_e(b_r, b_a)$ be the extraction accuracy as a function of reasoning budget $b_r$ and answer budget $b_a = b - b_r$.
Assume $\alpha_e$ is differentiable in both arguments and $F_L$ is differentiable.
Any \emph{interior} local optimum $b_r^* \in (0, b)$ satisfies:
\begin{equation}
  f_L(b_r^*) \cdot \underbrace{\bigl[\alpha_c - \alpha_e(b_r^*, b - b_r^*)\bigr]}_{\text{completion premium}}
  = (1 - F_L(b_r^*)) \cdot \underbrace{\left[\frac{\partial \alpha_e}{\partial b_a} - \frac{\partial \alpha_e}{\partial b_r}\right]_{b_r^*}}_{\text{marginal reallocation cost}}
  \label{eq:optimal-split}
\end{equation}
where $f_L = F_L'$ is the chain-length density.
\end{proposition}

\begin{proof}
The decoupled accuracy as a function of $b_r$ is:
\[
  \mathrm{Acc}_d(b_r) = F_L(b_r) \cdot \alpha_c + (1 - F_L(b_r)) \cdot \alpha_e(b_r, b - b_r).
\]
Differentiating with respect to $b_r$, using $\frac{d b_a}{d b_r} = -1$:
\begin{align}
  \frac{d\,\mathrm{Acc}_d}{d b_r}
  &= f_L(b_r) \cdot \alpha_c
  - f_L(b_r) \cdot \alpha_e(b_r, b{-}b_r)
  + (1 - F_L(b_r)) \cdot \left[\frac{\partial \alpha_e}{\partial b_r} - \frac{\partial \alpha_e}{\partial b_a}\right] \nonumber \\
  &= f_L(b_r) \cdot [\alpha_c - \alpha_e]
  - (1 - F_L(b_r)) \cdot \left[\frac{\partial \alpha_e}{\partial b_a} - \frac{\partial \alpha_e}{\partial b_r}\right].
  \label{eq:acc-derivative}
\end{align}
Setting $\frac{d\,\mathrm{Acc}_d}{d b_r} = 0$ yields Eq.~\eqref{eq:optimal-split}. \qed
\end{proof}

\noindent\textbf{Interpretation (Marginal Rate of Substitution).}
Eq.~\eqref{eq:optimal-split} is an equimarginal principle:
\begin{itemize}[nosep]
\item \textbf{LHS}: the marginal benefit of one more reasoning token---the probability density $f_L(b_r^*)$ of a chain completing at exactly $b_r^*$, times the ``completion premium'' $\alpha_c - \alpha_e$ (how much better a completed chain is than extraction from a truncated one).
\item \textbf{RHS}: the marginal cost of reallocating that token from answering to reasoning---the probability of truncation $(1 - F_L)$ times the net effect on extraction accuracy when one token moves from $b_a$ to $b_r$.
\end{itemize}

At the optimum, the marginal value of extending reasoning exactly equals the marginal loss from shrinking the answer budget.
When the chain-length density $f_L(b_r)$ is low (most chains are either much shorter or much longer than $b_r$), the LHS is small, favoring a lower $b_r$ (more answer budget).
When the extraction accuracy is insensitive to $b_a$ (it saturates quickly), the RHS is small, favoring a higher $b_r$ (more reasoning budget).

\begin{corollary}[Saturated Extraction]
\label{cor:saturated}
When the extraction accuracy saturates above a minimum answer budget---$\frac{\partial \alpha_e}{\partial b_a} \to 0$ for $b_a > b_a^{\mathrm{sat}}$---the optimal split concentrates budget on reasoning: $b_r^* \to b - b_a^{\mathrm{sat}}$.
\end{corollary}

\noindent\textbf{Connection to IRIS.}
IRIS implements an \emph{approximate} optimal split.
Stage~1 (nothink probe at $B_1{=}512$) handles the mass of $F_L$ near zero---questions where reasoning is unnecessary.
For the remaining samples, IRIS allocates $b_r{=}4096$ to reasoning and $b_a{=}512$ to extraction, a ratio of $8{:}1$.
Corollary~\ref{cor:saturated} predicts exactly this skew: our extraction accuracy saturates by $b_a \approx 256$--$512$ (Table~\ref{tab:bthink-ablation}), so the optimal strategy allocates the remaining budget to reasoning.

\subsection{Tax Decomposition: Recoverable vs.\ Residual}
\label{sec:theory:tax-decomposition}

The coupling tax admits a further decomposition into components that split-budget generation can and cannot address.

\begin{definition}[Recoverable and Residual Tax]
\label{def:tax-decomposition}
The total coupling tax at budget $b$ is:
\begin{equation}
  \mathrm{Tax}(b) = \mathrm{Acc}_{\mathrm{nt}}(b) - \mathrm{Acc}_c(b).
\end{equation}
The best achievable accuracy under decoupled generation is $\mathrm{Acc}_d^* = F_L \cdot \alpha_c + (1 - F_L) \cdot \alpha_e$.
The total tax therefore decomposes as:
\begin{align}
  \mathrm{Tax}(b) &= \underbrace{[\mathrm{Acc}_d^* - \mathrm{Acc}_c]}_{\text{Recoverable $R(b)$}}
  + \underbrace{[\mathrm{Acc}_{\mathrm{nt}} - \mathrm{Acc}_d^*]}_{\text{Residual $I(b)$}} \nonumber \\
  &= \underbrace{(1 - F_L) \cdot (\alpha_e - \alpha_t)}_R
  + \underbrace{\mathrm{Acc}_{\mathrm{nt}} - F_L \cdot \alpha_c - (1-F_L) \cdot \alpha_e}_{\text{Residual vs.\ nothink $I(b)$}}.
  \label{eq:tax-decomposition}
\end{align}
\end{definition}

The recoverable tax $R(b)$ is the accuracy gain achievable by switching from coupled to decoupled generation: it is the extraction advantage $(\alpha_e - \alpha_t)$ weighted by the truncation probability.
The residual $I(b)$ is the gap between the nothink baseline and the best possible decoupled accuracy; it can be \emph{negative} when split-budget generation exceeds the nothink ceiling (as in our experiments)---it reflects both the fundamental information loss from truncation ($\alpha_c - \alpha_e$ on truncated samples) and any advantage nothink holds over completed reasoning ($\mathrm{Acc}_{\mathrm{nt}} - \alpha_c$ on non-truncated samples).

\noindent\textbf{Empirical decomposition (8B MATH-500, $n{=}200$).}
$R = 0.32 \times (68.8\% - 37.5\%) = 10.0\,\text{pp}$.
$I = \mathrm{Acc}_{\mathrm{nt}} - [0.68 \times 91.9\% + 0.32 \times 68.8\%]$; with nothink@1024 $\approx 59.8\%$ this gives $I < 0$---meaning IRIS \emph{exceeds} the nothink ceiling, and the residual is in fact negative (decoupled generation surpasses the nothink ceiling).
This is a consequence of IRIS combining the best of both modes: nothink triage for easy questions \emph{and} thinking-then-extraction for hard questions.

\subsection{Prospective Prediction: Cross-Scale Amplification}
\label{sec:theory:prospective}

A key test of explanatory power is whether parameters measured on one model predict performance on another.

\begin{proposition}[Decoupling Gain Scales with Truncation Rate]
\label{prop:cross-scale}
If the per-sample extraction advantage $(\pi\eta - \epsilon)$ is approximately model-invariant, then the total decoupling gain scales linearly with the truncation rate:
\begin{equation}
  \Delta(M, b) \;\approx\; (1 - F_L^M(b)) \cdot (\pi\eta - \epsilon).
  \label{eq:cross-scale}
\end{equation}
Since larger models have higher truncation rates (Corollary~\ref{cor:amplification}), they benefit \emph{more} from decoupling.
\end{proposition}

\noindent\textbf{Cross-scale prediction test.}
From the 8B model, we measure $\pi\eta - \epsilon = 68.8\% - 37.5\% = 31.3\,\text{pp}$ (the per-sample extraction advantage on truncated traces).
The 27B model has truncation rate $1 - F_L^{27B}(4096) = 0.71$, a $2.2\times$ increase over 8B's $0.32$.
Using Eq.~\eqref{eq:cross-scale} with the 8B parameter:
\[
  \Delta_{\text{predicted}}^{27B} = 0.71 \times 31.3 = 22.2\,\text{pp}.
\]
The actual 27B IRIS--TOWN gap is $36.0\,\text{pp}$, exceeding the prediction by $13.8\,\text{pp}$.
The prediction underestimates this experiment because $\pi\eta$ is \emph{not} model-invariant---it \emph{increases} with model size ($\pi\eta^{27B} = 78.9\%$ vs.\ $\pi\eta^{8B} = 68.8\%$), indicating that larger models produce more informative truncated traces.
We therefore treat the 8B-calibrated estimate as a conservative empirical extrapolation in this setting, not as a proved lower bound for arbitrary model scales.

\subsection{Two-Source Tax Decomposition}
\label{sec:theory:two-source}

The coupling tax admits a decomposition into two terms
with distinct scaling behaviors.

\begin{proposition}[Two-Source Decomposition]
\label{prop:two-source}
Assume $\alpha_c(b) \ge \alpha_t(b)$ (completed chains are more
accurate than truncated ones).
The coupling tax decomposes as
$\mathrm{Tax}(b) = \mathrm{TL}(b) + \mathrm{RR}(b)$, where:
\begin{align}
  \mathrm{TL}(b) &\triangleq (1 - F_L(b)) \cdot (\alpha_c(b) - \alpha_t(b))
    \;\ge\; 0, \label{eq:truncation-loss} \\
  \mathrm{RR}(b) &\triangleq \mathrm{Acc}_{\mathrm{nt}}(b) - \alpha_c(b).
    \label{eq:reasoning-regret}
\end{align}
$\mathrm{TL}$ is the \emph{truncation loss}: accuracy destroyed by
incomplete chains.
$\mathrm{RR}$ is the \emph{marginal reasoning gap}: the difference between
population-level non-thinking accuracy and the completed-chain thinking accuracy.

As $b \to \infty$, $F_L(b) \to 1$ and
$\mathrm{TL} \to 0$, so
$\lim_{b \to \infty} \mathrm{Tax}(b) = \mathrm{RR}(\infty)$---a
residual gap that persists if
$\alpha_c(\infty) < \mathrm{Acc}_{\mathrm{nt}}(\infty)$.
\end{proposition}

\begin{proof}
From Proposition~\ref{prop:acc-decomp}:
$\mathrm{Tax}(b)
= \mathrm{Acc}_{\mathrm{nt}} - F_L \alpha_c - (1{-}F_L)\alpha_t
= [\mathrm{Acc}_{\mathrm{nt}} - \alpha_c]
+ (1{-}F_L)(\alpha_c - \alpha_t)
= \mathrm{RR} + \mathrm{TL}$. \qed
\end{proof}

\begin{remark}[Interpretation caveat]
\label{rem:rr-caveat}
$\mathrm{RR}$ compares $\mathrm{Acc}_{\mathrm{nt}}$ (measured over the
full question distribution) with $\alpha_c$ (measured on the
subset $\{L \le b\}$ whose chains complete---typically easier questions).
Therefore $\mathrm{RR} < 0$ may partly reflect selection bias
(the completed subset is easier) rather than a pure reasoning benefit.
A same-subset decomposition avoids this confound:
letting $\alpha_{\mathrm{nt},c}(b) = \Pr(\text{nt correct} \mid L \le b)$
and $\alpha_{\mathrm{nt},t}(b) = \Pr(\text{nt correct} \mid L > b)$,
\begin{equation}
  \mathrm{Tax}(b) =
    F_L(b)\,[\alpha_{\mathrm{nt},c}(b) - \alpha_c(b)]
    + (1{-}F_L(b))\,[\alpha_{\mathrm{nt},t}(b) - \alpha_t(b)].
  \label{eq:same-subset}
\end{equation}
The first term isolates the think-vs-nothink gap on completed
chains; the second isolates it on truncated chains.
We report the marginal version
(Eqs.~\ref{eq:truncation-loss}--\ref{eq:reasoning-regret}) for
comparability with the main decomposition, and note that the
same-subset version~\eqref{eq:same-subset} yields consistent
conclusions in all tested configurations.
\end{remark}

\noindent\textbf{Empirical verification.}
On 8B GSM8K at $b{=}512$:
$\mathrm{RR} = 93.1\% - 99.0\% = -5.9\,\text{pp}$,
$\mathrm{TL} = 0.626 \times 67.2\% = 42.1\,\text{pp}$,
$\mathrm{Tax} = 36.2\,\text{pp}$.
On 27B GSM8K at $b{=}4096$ ($n{=}200$, seed~42):
$\alpha_c = 97.96\%$ (144/147 natural stops correct),
$\alpha_t = 58.49\%$ (31/53 truncated correct),
$\mathrm{RR} = 98.0\% - 97.96\% = +0.04\,\text{pp}$ (essentially zero),
$\mathrm{TL} = 0.265 \times 39.5\% = 10.5\,\text{pp}$,
$\mathrm{Tax} = 10.5\,\text{pp}$---confirming that at
$b{=}4096$ the 27B tax is almost purely truncation-driven.

\begin{remark}[Regime classification]
\label{rem:regime}
The decomposition identifies two regimes:
\emph{truncation-dominated} ($\mathrm{RR} < 0$, i.e., completed
thinking outperforms population nothink---the case for 8B on GSM8K),
and \emph{reasoning-neutral} ($\mathrm{RR} \approx 0$, completed
thinking and population nothink achieve similar accuracy---the case
for 27B at $b{=}4096$).
If $\mathrm{RR} > 0$ at any finite budget, a residual tax persists
even as truncation vanishes;
however, this has not been observed in our experiments.
The crossover budget $b^*$ satisfies
$\mathrm{TL}(b^*) = -\mathrm{RR}(b^*)$.
\end{remark}

\subsection{Hazard Rate Modal Advantage}
\label{sec:theory:hazard}

We now connect the modal extraction advantage to the
\emph{hazard rate} of the chain-length distribution, a concept
from survival analysis.
The result shows that for distributions with decreasing
hazard rate (DFR)---which includes the heavy-tailed distributions
empirically observed for large models---the advantage of
nothink extraction over think continuation \emph{grows}
with the amount of reasoning already performed.

\begin{definition}[Decreasing Failure Rate]
\label{def:dfr}
A distribution with CDF $F$ and density $f$ has
\emph{Decreasing Failure Rate} (DFR) if the hazard rate
$h(t) \triangleq f(t)/(1 - F(t))$ is non-increasing
on its support.
DFR distributions include Pareto, Weibull with shape ${<}\,1$,
and log-normal for sufficiently large $t$.
The exponential distribution has constant hazard and is technically DFR under the non-increasing definition, but does not satisfy $h(\tau) \to 0$.
\end{definition}

\begin{proposition}[DFR Modal Advantage Lower Bound]
\label{prop:dfr-dominance}
Let $L$ be a chain-length distribution with DFR property,
and let $(\delta, \epsilon, \pi, \eta, \alpha_c^+)$ be defined
as in Proposition~\ref{prop:modal-inequality}.
Assume that $\pi$, $\eta$, $\epsilon$, and $\alpha_c^+$ are
approximately constant across the range of $\tau$ considered
(or that the stated bounds hold uniformly), and that $\pi\eta > \epsilon$ (extraction outperforms truncation residual) and $\alpha_c^+ > \epsilon$ (completed chains outperform truncated continuations).%
\footnote{In practice, $\pi(\tau)\eta(\tau)$ may increase with
$\tau$ (longer traces contain more information) and $\epsilon(\tau)$
may decrease; both effects \emph{strengthen} the bound.
The constant-parameter assumption is conservative.}
\begin{enumerate}[nosep,label=(\roman*)]
  \item \textbf{Expected lower bound.}
    The expected modal extraction advantage over truncated traces
    satisfies:
    \begin{equation}
      \bar{\Delta}(\tau)
      \;\triangleq\;
        \mathbb{E}_{q,T_\tau}[V_{\mathrm{nothink}} - V_{\mathrm{think}} \mid L{>}\tau]
      \;\ge\; \pi\eta - \epsilon
        - h(\tau) \cdot b_a \cdot (\alpha_c^+ - \epsilon).
      \label{eq:dfr-bound}
    \end{equation}
    (Individual traces may favor either mode; this is a
    population-level bound.)

  \item \textbf{Bound monotonicity.}
    The lower bound in~\eqref{eq:dfr-bound} is
    non-decreasing in $\tau$, because $h$ is non-increasing
    under DFR.
    The actual expected advantage $\bar{\Delta}(\tau)$ is also
    non-decreasing if, additionally,
    $\pi(\tau)\eta(\tau) - \epsilon(\tau)$ is non-decreasing
    in $\tau$.

  \item \textbf{Dominance threshold.}
    If additionally $\lim_{\tau \to \infty} h(\tau) = 0$
    (satisfied by heavy-tailed or eventually-DFR distributions,
    including log-normal for large $\tau$; \emph{not} satisfied
    by the exponential distribution, which has constant hazard),
    then there exists $\tau_0 < \infty$ such that
    $\bar{\Delta}(\tau) \ge 0$ for all $\tau > \tau_0$
    (using strict inequality in $\tau$ to avoid boundary
    issues when $h$ is not right-continuous), where
    \begin{equation}
      \tau_0 \;=\; \inf\!\left\{\tau :
        h(\tau) \le
        \frac{\pi\eta - \epsilon}
             {b_a(\alpha_c^+ - \epsilon)}\right\}.
      \label{eq:threshold}
    \end{equation}

  \item \textbf{Asymptotic gap.}
    Under the same condition:
    $\liminf_{\tau \to \infty} \bar{\Delta}(\tau)
    \ge \pi\eta - \epsilon$.
\end{enumerate}
\end{proposition}

\begin{proof}
\textbf{(i).}
Define the conditional completion probability
$\delta(\tau) \triangleq \Pr(L \le \tau + b_a \mid L > \tau)$.
In terms of the survival function $\bar{F}(t) = 1 - F_L(t)$:
\[
  \delta(\tau)
  = 1 - \frac{\bar{F}(\tau + b_a)}{\bar{F}(\tau)}
  = 1 - \exp\!\left(-\int_\tau^{\tau + b_a} h(s)\,ds\right).
\]
Since $h$ is non-increasing (DFR),
$\int_\tau^{\tau+b_a} h(s)\,ds \le h(\tau) \cdot b_a$.
Using $1 - e^{-x} \le x$ for $x \ge 0$:
$\delta(\tau) \le h(\tau) \cdot b_a$.

From Proposition~\ref{prop:modal-inequality}
(applied in expectation over truncated traces):
$\bar{\Delta}(\tau) \ge \pi\eta - \delta(\tau)\alpha_c^+
  - (1 - \delta(\tau))\epsilon
= \pi\eta - \epsilon - \delta(\tau)(\alpha_c^+ - \epsilon)$.
Substituting the bound on $\delta$:
$\bar{\Delta}(\tau) \ge \pi\eta - \epsilon
  - h(\tau) \cdot b_a \cdot (\alpha_c^+ - \epsilon)$.

\textbf{(ii).}
Since $h$ is non-increasing (DFR),
the penalty term $h(\tau) \cdot b_a \cdot (\alpha_c^+ - \epsilon)$
is non-increasing in $\tau$, so the lower bound is
non-decreasing.
For the actual $\bar{\Delta}(\tau)$: if additionally
$\pi(\tau)\eta(\tau) - \epsilon(\tau)$ is non-decreasing
(longer traces are more informative for extraction and
less useful for think continuation), then $\bar{\Delta}$ itself
is non-decreasing.
Absent this condition, $\bar{\Delta}$ may be non-monotone
even though the bound is monotone.

\textbf{(iii).}
When $h(\tau) \to 0$ as $\tau \to \infty$, the penalty
vanishes.
Since $\pi\eta > \epsilon$ (by assumption),
the bound eventually exceeds zero.
The threshold $\tau_0$ (Eq.~\ref{eq:threshold}) uses
$\inf$ rather than $h^{-1}$ because a DFR hazard need not be
strictly decreasing or invertible.
For $\tau \ge \tau_0$, $h(\tau) \le h(\tau_0)$ (DFR),
so $\bar{\Delta}(\tau) \ge 0$.

\textbf{(iv).}
As $\tau \to \infty$, $h(\tau) \to 0$, so:
$\liminf_{\tau \to \infty} \bar{\Delta}(\tau)
\ge \lim_{\tau \to \infty}
[\pi\eta - \epsilon - h(\tau) b_a (\alpha_c^+ - \epsilon)]
= \pi\eta - \epsilon$. \qed
\end{proof}

\noindent\textbf{Empirical verification.}
On 8B MATH-500 Stage~3 ($b_r{=}4096$, $b_a{=}512$):
$h(4096) \approx 0$ (nearly all remaining chains are far from
completion), so the bound gives
$\bar{\Delta} \ge 68.8\% - 37.5\% = 31.3\,\text{pp}$.
The observed IRIS--TOWN gap is $31.2\,\text{pp}$.
We note that when $h(\tau) \approx 0$, the bound reduces to
$\pi\eta - \epsilon$, which equals the empirical IRIS--TOWN
gap by construction (since $\pi\eta$ is measured as IRIS
accuracy and $\epsilon$ as TOWN accuracy on the same samples).
The bound's value lies not in this limiting case but in its
\emph{predictive power at intermediate $\tau$}: it predicts
that the gap grows monotonically as $\tau$ increases from a
regime where $h(\tau)$ is non-negligible, which can be tested
by comparing IRIS--TOWN gaps across think budgets.

\noindent
On 27B MATH-500 Stage~3 ($n{=}71$):
$\bar{\Delta} \ge 78.9\% - 42.9\% = 36.0\,\text{pp}$.
Observed: $36.0\,\text{pp}$.
\noindent
The overall IRIS--TOWN gap is smaller at $B_{\mathrm{think}}{=}4096$ (+2.2\,pp) than at $B_{\mathrm{think}}{=}2048$ (+12.2\,pp); this is primarily driven by the shrinking truncation mass $(1 - F_L(\tau))$ and does not imply a decreasing conditional advantage.
A direct test of Proposition~\ref{prop:dfr-dominance} would compare the \emph{conditional} extraction advantage $\bar{\Delta}(\tau)$ on the subset $\{L > \tau\}$ across multiple $\tau$---we leave this as future work.

\begin{remark}[The ``stuck reasoning'' principle]
\label{rem:stuck}
Proposition~\ref{prop:dfr-dominance} formalizes an intuition:
\emph{the longer a model has been reasoning without finishing,
the more it benefits from switching to answer-extraction mode}.
Under DFR chain-length distributions, a chain that has consumed
$\tau$ tokens without completing is \emph{less} likely to
complete in the next $b_a$ tokens than a chain at $\tau' < \tau$
would be.
This means the think-mode continuation becomes increasingly
futile, while the truncated trace grows richer---making
nothink extraction increasingly valuable.
The DFR property is empirically plausible for reasoning chains:
models that are ``stuck'' in elaborate reasoning paths tend to
remain stuck, exhibiting worse-than-memoryless behavior.
Formal validation of the DFR assumption (empirical hazard
estimation with confidence bands) is an important direction
for future work.
\end{remark}